\documentclass{article}

\usepackage[main, final]{neurips_2026}

\usepackage[utf8]{inputenc}
\usepackage[T1]{fontenc}
\usepackage{hyperref}
\usepackage{url}
\usepackage{booktabs}
\usepackage{amsfonts}
\usepackage{nicefrac}
\usepackage{microtype}
\usepackage{xcolor}

\usepackage{graphicx}
\usepackage{bm}
\usepackage{amsmath}
\usepackage{mathtools}
\usepackage{subcaption}
\usepackage{enumitem}
\usepackage{placeins}
\usepackage{booktabs}
\usepackage{multirow}
\usepackage[table]{xcolor}

\usepackage{amssymb}
\usepackage{amsthm}
\theoremstyle{theorem}

\newtheorem{proposition}{Proposition}

\theoremstyle{definition}
\newtheorem{definition}{Definition}

\usepackage{algorithm}
\usepackage{algpseudocodex}

\newcommand{\norm}[1]{\left\lVert #1 \right\rVert}

\title{Parallel-in-Time Training of Recurrent Neural Networks for Dynamical Systems Reconstruction}

\author{
    Florian Hess\textsuperscript{1,2\dag} \hspace{3em}
    Florian Götz\textsuperscript{1,3} \hspace{3em}
    Daniel Durstewitz\textsuperscript{1,2,4,5\dag}\vspace{1.5em} \\
    \textsuperscript{1}Dept. of Theoretical Neuroscience, Central Institute of Mental Health,
    Mannheim, Germany \\
    \textsuperscript{2}Faculty of Physics and Astronomy, Heidelberg University, Germany \\
    \textsuperscript{3}Faculty of Mathematics and Computer Science, Heidelberg University, Germany \\
    \textsuperscript{4}Interdisciplinary Center for Scientific Computing (IWR), Heidelberg University, Germany\\
    \textsuperscript{5}ELLIS Unit Heidelberg, Germany
}

\begin{document}

\maketitle
\begingroup
\renewcommand{\thefootnote}{\dag}%
\footnotetext{Correspondence to \href{mailto:florian.hess@zi-mannheim.de}{florian.hess@zi-mannheim.de} and \href{mailto:daniel.durstewitz@zi-mannheim.de}{daniel.durstewitz@zi-mannheim.de}.}
\renewcommand{\thefootnote}{\arabic{footnote}}
\setcounter{footnote}{0}
\endgroup
\begin{abstract}
Reconstructing nonlinear dynamical systems (DS) from data (DSR) is a fundamental challenge in science and engineering, but it inherently relies on \textit{sequential} models. Recent breakthroughs for sequential models have produced algorithms that parallelize computation along sequence length $T$, achieving logarithmic time complexity, $\mathcal{O}(\log T)$. Since sequence lengths have been practically limited due to the linear runtime complexity $\mathcal{O}(T)$ of classical backpropagation through time, this opens new avenues for DSR. This paper studies two prominent classes of parallel-in-time algorithms for this task, both of which leverage \textit{parallel associative scans} as their core computational primitive. The first class comprises models with linear yet non-autonomous dynamics and a nonlinear readout, such as modern State Space Models (SSMs), while the second consists of general nonlinear models which can be parallelized using the DEER framework. We find that the linear training-time recurrence of the first class of models imposes limitations that often hinder learning of accurate nonlinear dynamics. To address this, we augment DEER with Generalized Teacher Forcing (GTF), a novel variant within the more general nonlinear framework that ensures stable and effective learning of nonlinear dynamics across arbitrary sequence lengths. Using GTF-DEER, we investigate the benefits of training on extremely long sequences ($T$ approaching $10^5$) for DSR. Our results show that access to such long trajectories significantly improves DSR if the data features long time scales. This work establishes GTF-DEER as a robust tool for data-driven discovery and underscores the largely untapped potential of long-sequence learning in modeling complex DS.\footnote{Code to reproduce the results of this work is available at \href{https://github.com/DurstewitzLab/Parallel-In-Time-Training-DSR}{https://github.com/DurstewitzLab/Parallel-In-Time-Training-DSR}.}
\end{abstract}

\section{Introduction}
Understanding and predicting the behavior of complex nonlinear systems from neural circuits and climate dynamics to fluid flows and ecological networks is a central aim across the natural and engineering sciences \citep{BhallaIyengar1999Emergent, izhikevich_dynamical_2007, Kalnay2003AMDAP, tziperman-97}. A particularly ambitious goal is to \emph{reconstruct} the underlying dynamical system (DS) directly from observed time series, a problem known as dynamical systems reconstruction (DSR). Beyond short-term forecasting, DSR requires that the learned model faithfully reproduces the long-term statistical and geometric properties of the true system, such as attractor geometry, power spectra, and Lyapunov exponents.

Central to the definition of DS is the flow operator, which provides a recursive rule of how the DS evolves in time \cite{katok_introduction_1995, strogatz2024nonlinear}. DSR methods that approximate it are inherently recursive~\citep{goring_domain_2024,mikhaeil_difficulty_2022}, such as RNNs, commonly trained by backpropagation through time (BPTT; \citep{werbos1990backpropagation}). While several well-documented pathologies of BPTT -- most notably exploding gradients under chaotic dynamics \citep{mikhaeil_difficulty_2022} -- can be successfully mitigated by control-theoretic training algorithms such as sparse and generalized teacher forcing (STF/GTF) \citep{hess_generalized_2023, mikhaeil_difficulty_2022}, the computational cost of BPTT remains fundamentally linear in the sequence length $T$, rendering training on problems with long intrinsic timescales prohibitively expensive. Indeed, DSR applications have historically been confined to modest sequence lengths to keep training tractable \citep{brenner_tractable_2022, hess_generalized_2023, vlachas2024learning}, leaving open the question of whether learning from substantially longer sequences offers any benefit for reconstruction quality.

Recent advances in parallel-in-time sequence modeling offer a path forward, and we examine two common paradigms with respect to their performance in DSR. Linear training-time recurrences with a nonlinear readout, as instantiated by modern SSMs, admit trivial parallelization via linear scans and avoid chaos-induced exploding gradients by construction \citep{mikhaeil_difficulty_2022, orvieto2023resurrecting, zucchet2024recurrent}. By exposing a \textit{duality} between these linear SSMs and nonlinear RNNs, we show, however, that this convenience comes at a cost: the linear recurrence used during training imposes structural limitations that often prevent the model from learning accurate nonlinear dynamics. Moreover, its training suffers from exposure bias, degrading autoregressive roll-outs at test-time \citep{bengio2015scheduled, ranzato2016sequence}. The second paradigm -- general nonlinear RNNs parallelized via DEER \citep{lim2024parallelizing} -- is in principle better suited to DSR, but naive application fails on chaotic data because the Jacobian products driving DEER's Newton updates diverge whenever the underlying dynamics exhibits positive Lyapunov exponents \citep{gonzalez2025predictability, mikhaeil_difficulty_2022}.

Our main methodological contribution is to resolve this tension by combining DEER with Generalized Teacher Forcing (GTF) \citep{hess_generalized_2023}. The resulting algorithm, \textbf{GTF-DEER}, inherits the long-sequence scalability of DEER while retaining GTF's chaos-taming properties: GTF turns the model into a stable DS during training such that GTF-DEER enjoys the average-case $\mathcal{O}((\log T)^2)$  scaling established in \cite{gonzalez2025predictability}, even when the underlying system is chaotic. Empirically, GTF-DEER delivers speedups of up to $870\times$ over sequential training while matching or improving reconstruction quality. Equipped with GTF-DEER, we then revisit the central empirical question: \emph{does training on much longer sequences actually improve DSR?} Leveraging the ability to train stably on trajectories of length $T$ approaching $10^5$, we find that access to long trajectories allows the DSR model to efficiently capture long time scales if present in the data, yielding substantial gains in long-term statistics, which an equivalent linear SSM based model cannot match. Taken together, our results establish GTF-DEER as a direct replacement of sequential training while highlighting long-sequence training as a largely untapped lever for DSR.

\section{Related work}
\paragraph{Dynamical systems reconstruction.} Data-driven methods for DSR fall into two broad categories: The first approximates the \textit{vector field} that underlies the observed dynamics, assuming the process is governed by differential equations. Sparse Identification of Nonlinear Dynamics (SINDy) and its variants \citep{brunton_data-driven_2019, brunton_discovering_2016, champion_data-driven_2019, cortiella_sparse_2021, messenger_weak_2021} are particularly popular in the physical sciences and enable quick training through least-squares regression, but rely on pre-defined function libraries and struggle with noisy and partially observed or empirical data more generally. Neural ODE/PDE methods \citep{aka2025balanced, alvarez_dynode_2020, chen_neural_2018, karlsson_modelling_2019, ko_homotopy-based_2023} offer universal approximation capabilities and can be augmented with physical priors \citep{li_physics-informed_2022, raissi_physics-informed_2019}, but are difficult to train in practice. The second and larger class of models directly approximates the \textit{flow operator} through universal approximators based on neural networks, including neural operators \citep{li2020fourier, lu_learning_2021}, Koopman operators \citep{azencot_forecasting_2020, brunton_modern_2021, lusch_deep_2018, naiman_koopman_2021, otto_linearly-recurrent_2019, wang_koopman_2022} and hybrids thereof, reservoir computers \citep{gauthier_next_2021, patel_using_2023, pathak_model-free_2018, pathak_using_2017, platt2023constraining, verzelli2021learn}, and RNNs trained through BPTT \citep{brenner_almost_2024, brenner_tractable_2022, hess_generalized_2023, rusch_long_2022, vlachas_data-driven_2018, vlachas2024learning}, often accompanied by specialized control techniques that ease optimization and address exploding gradients under chaos \citep{brenner_tractable_2022, hess_generalized_2023, mikhaeil_difficulty_2022, sagtekin2025error}. The latter approach in particular achieves SOTA performance on a wide range of benchmark systems and performs well even on challenging empirical systems \citep{brenner_integrating_2024,hess_generalized_2023,volkmann2024scalable}, which has led to its adoption as a backbone in foundation models for DSR \citep{hemmer2025true}.

\paragraph{Efficient sequence modeling.} A central challenge in sequence modeling is to stably capture and retrieve long-term dependencies \citep{bengio_learning_1994, gu2024mamba, gu_efficiently_2022, hochreiter_lstm_97,
orvieto2023resurrecting, zucchet2024recurrent}. One practical bottleneck when using autoregressive models such as SSMs or RNNs is the $\mathcal{O}(T)$ runtime of BPTT \citep{werbos1990backpropagation}. Recent developments in the field of sequence models address this issue by parallelizing the inherently sequential operation on parallel accelerators, such as GPUs or TPUs, using \textit{parallel associative scans} \citep{blelloch1990prefix, martin2018parallelizing, smith_simplified_2023}. Parallel scans evaluate \emph{linear} recurrences, often used as the core block of modern SSMs \citep{gu2024mamba, orvieto2023resurrecting, smith_simplified_2023}, in $\mathcal{O}(\log T)$ time. Furthermore, \citep{lim2024parallelizing} introduced the DEER framework which parallelizes general \textit{nonlinear} sequence models by reformulating the forward pass as a fixed-point iteration problem, where each iteration can be solved using a parallel scan. While the worst-case runtime complexity is $\mathcal{O}(T\log T)$ \citep{gonzalez2024towards}, the average-case observed in practice for models exhibiting contracting dynamics scales as $\mathcal{O}((\log T)^2)$ \cite{gonzalez2025unifying, gonzalez2025predictability}, sparking a renaissance for nonlinear RNNs in long-term sequence modeling. Although DEER demonstrates promising performance in various ML problems, its applicability to the field of DSR remains impractical due to unavoidable worst-case scaling when evaluating chaotic RNNs.

\section{Theoretical background}

\subsection{Dynamical systems reconstruction with autoregressive models}
\label{subsec:autoregressive_dsr}
Given observed time series data $\bm{X} \in \mathbb{R}^{T_{\mathrm{obs}}\times N}$ originating from some underlying physical process, DSR seeks to learn a generative model that is able to both perform accurate short-term predictions and reproduce the long-term behavior of the observed system. To tackle this task, we consider parameterized state space models of the form
\begin{equation}\label{eq:general_stsp_model}
\begin{aligned}
        \bm{z}_t &= F_{\boldsymbol{\theta}}(\bm{z}_{t-1}, \bm{x}_{t-1}, \bm{s}_t) \\
    \bm{\hat{x}}_t &= G_{\boldsymbol{\psi}}(\bm{z}_t),
\end{aligned}
\end{equation}
where $\boldsymbol{\theta}$ and $\boldsymbol{\psi}$ denote parameter vectors, $\bm{z}_t$ is an $M$-dimensional state vector, $\bm{\hat{x}}_t$ are $N$-dimensional predicted observations, $\bm{x}_{t-1}$ is an optional teacher signal used during training, and $\bm{s}_t$ are $K$-dimensional optional external inputs. $F_{\boldsymbol{\theta}}$ is a discrete-time universal DS modeling a latent process which is coupled to the observations (data) through $G_{\boldsymbol{\psi}}$. The aim of training is to learn $\{\bm{\theta}, \bm{\psi}\}$ such that the model approximates the flow operator of the underlying DS, and after training we have 
\begin{equation}
    \bm{x}_t \approx G_{\bm{\psi}}(F_{\bm{\theta}}(\dots F_{\bm{\theta}}(F_{\bm{\theta}}(\bm{z}_0, \bm{s}_1), \bm{s}_2)\dots,\bm{s}_t)) \eqqcolon G_{\bm{\psi}}(F_{\bm{\theta}}^{\circ t}(\bm{z}_0, \bm{s}_{1:t})). 
\end{equation}

In this work, we will investigate two general parameterizations of Eq.~\eqref{eq:general_stsp_model}. 

\subsubsection{Linear training-time recurrences}\label{subsec:linear_recurrences}
In one setting, $F_{\bm{\theta}}$ is strictly \textit{linear}, which shifts the burden of capturing nonlinearities in the data to a nonlinear observation function $G_{\boldsymbol{\psi}}$ which feeds back into the system at test time, 
\begin{equation}\label{eq:shplrnn_ssm}
\begin{aligned}
        \bm{z}_t &= \bm{A}\bm{z}_{t-1} + \bm{U}\bm{x}_{t-1} + \bm{C}\bm{s}_t + \bm{h} \\
    \hat{\bm{x}}_t &= \bm{B}\phi(\bm{V}\bm{z}_t+\bm{b}),
\end{aligned}
\end{equation}
where $\bm{A} \in \mathbb{R}^{M\times M}$, $\bm{U} \in \mathbb{R}^{M\times N}$, $\bm{C} \in \mathbb{R}^{M \times K}$, $\bm{h} \in \mathbb{R}^M$, $\bm{B} \in \mathbb{R}^{N \times L}$, $\bm{V}\in\mathbb{R}^{L \times M}$, $\bm{b}\in\mathbb{R}^L$ and $\phi$ is a nonlinear function such as the $\mathrm{ReLU}(\bm{z}) = \mathrm{max}(0, \bm{z})$. The recurrence matrix $\bm{A}$ is usually parameterized as a \emph{diagonal} matrix to boost computational efficiency and facilitate optimization \citep{gu2024mamba, orvieto2023resurrecting, zucchet2024recurrent}, at the cost of a decrease in dynamical repertoire \citep{muca2024theoretical, ran2024provable}. Eq.~\eqref{eq:shplrnn_ssm} is essentially the architectural setup of modern SSMs, where linear recurrences are followed by non-linear, point-wise sequence transformations \citep{gu2024mamba, orvieto2024universality, orvieto2023resurrecting, smith_simplified_2023}, often implemented by MLPs. In our case, the MLP is a simple one-hidden-layer neural network with hidden layer size $L$. In the following, we will refer to models defined by Eq.~\eqref{eq:shplrnn_ssm} as `LSSM'. A crucial insight is that while the recurrence in Eq.~\eqref{eq:shplrnn_ssm} constitutes a linear, \textit{non-autonomous} DS during training, the model can produce nonlinear dynamics during evaluation by replacing the teacher signals $\bm{x}_{t-1}$ with predictions $\hat{\bm{x}}_{t-1}$. Indeed, during trajectory generation \textit{after} training, the recursion of Eq.~\eqref{eq:shplrnn_ssm} turns into
\begin{equation}\label{eq:shplrnn_ssm_test_time}
    \bm{z}_t = \bm{A}\bm{z}_{t-1} + \overline{\bm{W}}\phi(\bm{V}\bm{z}_{t-1} + \bm{b}) + \bm{C}\bm{s}_t + \bm{h}
\end{equation}
where $\bm{U}\bm{B} =: \overline{\bm{W}} \in\mathbb{R}^{M\times L}$ with $\operatorname{rank}(\overline{\bm{W}}) \leq \min(N, M, L)$. This allows (L)SSMs to exhibit inherently nonlinear traits such as chaotic dynamics and multistability during autoregressive generation. While the major feature of Eq.~\eqref{eq:shplrnn_ssm} is the fact that the forward pass can be calculated in logarithmic time using parallel scan, another insight is that linear recurrences do not suffer from chaos-induced exploding gradients \citep{mikhaeil_difficulty_2022}, and hence remedy associated training instabilities by design \citep{mikhaeil_difficulty_2022, orvieto2023resurrecting, zucchet2024recurrent}; see Appx. \ref{appx:sec:no_exploding_grads_for_ssms} for details. Training of LSSMs makes use of a variant of teacher forcing originating from the sequence modeling field, where ground-truth data from the previous time step $\bm{x}_{t-1}$ is fed through a dedicated input layer ($\bm{U}$) \citep{bengio2015scheduled, Goodfellow-et-al-2016}; a generic training algorithm is provided in Alg. \ref{alg:LSSM}. This type of training suffers from \textit{exposure bias}, where the absence of forcing signals during evaluation (Eq.~\eqref{eq:shplrnn_ssm_test_time}) significantly degrades autoregressive roll-outs \citep{ranzato2016sequence}. While methods such as scheduled sampling address this issue \citep{bengio2015scheduled, vlachas2024learning}, they destroy linearity and therefore efficient parallelization of the recurrence in Eq.~\eqref{eq:shplrnn_ssm} during training; see Appx.~\ref{appx:sec:scheduled_sampling_breaks_scan} for details.

\subsubsection{Non-linear training-time recurrences}\label{subsec:nonlinear_recurrences}
To compare training algorithms, we will consider a second parameterization of Eq.~\eqref{eq:general_stsp_model}, which uses a nonlinear RNN as the latent model, coupled to the observations through a simple linear mapping:
\begin{equation}\label{eq:shplrnn_nonlinear}
\begin{aligned}
       \bm{z}_t &= \bm{A}\bm{z}_{t-1} + \bm{W}\phi(\bm{V}\bm{z}_{t-1}+\bm{b}) + \bm{C}\bm{s}_t +\bm{h} \\
       \hat{\bm{x}}_t &= \bm{B}\bm{z}_t.
\end{aligned}
\end{equation}

With $\phi(\cdot)= \mathrm{ReLU}(\cdot)$, the latent model is known as a shallow piecewise-linear RNN (shPLRNN; \citep{hess_generalized_2023}), an established RNN architecture designed for DSR. Indeed, up to the constraint of a low-rank connectivity matrix, which can be explicitly enforced in the latter parameterization, the recurrences of Eqs.~\eqref{eq:shplrnn_ssm_test_time} and \eqref{eq:shplrnn_nonlinear} are \textit{mathematically equivalent}. This is important as it allows us to \textit{directly compare training algorithms and parameterizations, unconfounded} by any architectural differences at test time, studying optimized models with parameters $\bm{\theta}_{\mathrm{RNN}} \cup \bm{\psi}_{\mathrm{RNN}} = \{\bm{A}, \bm{W}, \bm{V}, \bm{b}, \bm{C}, \bm{h}, \bm{B}\}$ and $\bm{\theta}_{\mathrm{LSSM}} \cup \bm{\psi}_{\mathrm{LSSM}}=\{\bm{A}, \overline{\bm{W}}, \bm{V}, \bm{b}, \bm{C}, \bm{h}, \bm{B}\}$. All results established below hold in the presence of external inputs $\bm{s}$, but we will drop them from notation for brevity.

Due to the nonlinearity of the latent model in Eq.~\eqref{eq:shplrnn_nonlinear}, the forward pass is evaluated sequentially and cannot be naively parallelized. Hence, a common method to train for DSR is to generate trajectories of length $T$ and then compare them to the data $\bm{x}_{1:T}$. To avoid training instabilities caused by, for example, exploding gradients, training is accompanied by control-theoretic forcing methods such as Generalized Teacher Forcing (GTF; \citep{doya_bifurcations_1992, hess_generalized_2023}). GTF alters the forward pass of the model during training by linearly interpolating between the latent state and a teacher signal before application of the sequence model:
\begin{equation}\label{eq:GTF}
\begin{aligned}
    \bm{z}_{t} = F_{\bm{\theta}} \circ \delta_\alpha(\bm{z}_{t-1}, \overline{\bm{z}}_{t-1}) = F_{\bm{\theta}}((1-\alpha)\bm{z}_{t-1} + \alpha\overline{\bm{z}}_{t-1})
    = F_{\bm{\theta}}(\tilde{\bm{z}}_{t-1}),
\end{aligned}
\end{equation}
where $0 \le \alpha \le 1$ is the forcing strength and $\overline{\bm{z}}$ denotes the teacher signal which is generally computed by inversion of the observation model or output layer, i.e. $\overline{\bm{z}}_t=G_{\boldsymbol{\psi}}^{-1}(\bm{x}_t)$,\footnote{The observation model does not need to be strictly invertible and often an approximate inverse, such as the pseudo-inverse in case of linear or affine $G_{\bm{\psi}}$, is used \citep{hess_generalized_2023, mikhaeil_difficulty_2022}. Alternatively, the latent model may be embedded in an encoder-decoder architecture, where forcing signals are given through the encoder mapping \citep{brenner_integrating_2024}.} and $\tilde{\bm{z}}$ is the forced state. During training, this leads to decomposition of the Jacobian as
\begin{equation}\label{eq:jacobian_decomp}
    \bm{J}_{F\circ\delta}(\bm{z}_{t-1}) =  (1-\alpha)\bm{J}_F(\tilde{\bm{z}}_{t-1}),
\end{equation}
where $\bm{J}_{F}(\bm{z}) \coloneqq \frac{\partial{F(\bm{z})}}{\partial{ \bm{z}}}$. The forcing parameter $\alpha \in [0,1]$ controls the norm of the Jacobians during BPTT and can be optimally and adaptively adjusted in training to mitigate exploding gradients (\citep{hess_generalized_2023}, see also Appx.~\ref{appx:alpha_choice}). 

For linear observation models (Eq.~\eqref{eq:shplrnn_nonlinear}) and the common case where the RNN has more units than there are observed dynamical variables ($M > N$), \citep{sagtekin2025error} introduced a correction to GTF \eqref{eq:GTF} which includes the row-space projector of $\bm{B}$ in the forcing equation:
\begin{equation}\label{eq:error_forcing}
    \delta_{\alpha,\bm{B}}(\bm{z}_t, \bm{x}_t) = (\bm{I} - \alpha \bm{B}^+\bm{B})\bm{z}_t + \alpha \bm{B}^+\bm{x}_t,
\end{equation}
where $\bm{B}^+$ denotes the pseudo-inverse of $\bm{B}$, such that Eq.~\eqref{eq:jacobian_decomp} becomes
\begin{equation}\label{eq:error_forcing_decomp}
    \bm{J}_{F\circ\delta_{\alpha, \bm{B}}}(\bm{z}_{t-1}) =  \bm{J}_F(\tilde{\bm{z}}_{t-1})\bm{P}_\alpha,
\end{equation}
where $\bm{P}_\alpha := \bm{I} - \alpha \bm{B}^+\bm{B}$. While this correction does not preserve the mitigation of exploding gradients in general, it dampens Jacobian singular values along directions corresponding to expansion in observation space. Indeed, we find empirically that Jacobian damping still holds in practice in this case, which we attribute to mixing of directions through the Jacobian product. Note that for $M \leq N$ and assuming full column-rank, $\bm{B}^+\bm{B} = \bm{I}$ and vanilla GTF is recovered. For a detailed overview of the training protocol under GTF, see Alg.~\ref{alg:ivp}.

\subsubsection{Parallel-in-time algorithms for nonlinear sequence models}
Training nonlinear flow operators on long sequences sequentially is prohibitive due to the linear $\mathcal{O}(T)$ scaling of BPTT in sequence length. To make training feasible, we will make use of the recently proposed DEER algorithm \citep{gonzalez2024towards, lim2024parallelizing}. Given an initial condition $\bm{z}_0$, let $\bm{z}_{1:T}$ be a series of candidate latent states and define the residual vector $\bm{r}(\bm{z}_{1:T}) := [\bm{z}_1 - F(\bm{z}_{0}), \dots,\bm{z}_T - F(\bm{z}_{T-1})] \in \mathbb{R}^{MT}$. The fixed point $\bm{z}_{1:T}^\ast=F(\bm{z}_{0:T-1}^\ast)$ is the only solution with zero residual, i.e. $\bm{r}(\bm{z}_{1:T}^\ast) = \bm{0}$. Conventionally, the forward pass is generated sequentially by iterating $F$ starting from $\bm{z}_0$. DEER turns the roll-out of the map $F$ into a root-finding problem for the residual, which is solved using Newton's method.
Starting from an initial guess $\bm{z}_{1:T}^{(0)}$, where the superscript indicates the current Newton iteration, the true trace $\bm{z}_{1:T}^\ast$ is approximated by iteratively solving the update equation 
\begin{equation}\label{eq:DEER_vector_update}
    \bm{z}_{1:T}^{(i+1)} = \bm{z}_{1:T}^{(i)} - \left(\bm{J}_{\bm{r}}(\bm{z}_{1:T})\right)^{-1} \bm{r}\left(\bm{z}_{1:T}^{(i)}\right),
\end{equation}
where $\bm{J}_{\bm{r}} := \frac{\partial \bm{r}}{\partial \bm{z}_{1:T}}$. Inverting the Jacobian matrix explicitly is infeasible for large $MT$. Instead, it is more practical to multiply both sides of the equation by the Jacobian and exploit its block-bidiagonal structure (cf.\ Eq.~\eqref{eq:blockbidiag_jacobian} in Appx.~\ref{appx:backward_pass}) to obtain a recursive formula for $\Delta\bm{z}_t^{(i+1)} := \bm{z}_t^{(i+1)} - \bm{z}_t^{(i)}$:

\begin{equation}\label{eq:DEER_update}
    \Delta \bm{z}_t^{(i+1)} = \bm{J}_{F}\left(\bm{z}_{t-1}^{(i)}\right) \Delta \bm{z}_{t-1}^{(i+1)} - \left[\bm{r}(\bm{z}_{1:T}^{(i)})\right]_t
\end{equation}

Equation \eqref{eq:DEER_update} is linear in $\Delta \bm{z}$ and hence each Newton iteration can be solved in $\mathcal{O}(\log T)$ time by a parallel associative scan. Due to its usage of the full Jacobian and hence full matrix products, increasing the dimensionality of the state of the sequence model can become a bottleneck. To address this, \cite{gonzalez2024towards} suggested the use of quasi-Newton methods which replace the full Jacobian in Eq.~\eqref{eq:DEER_update} by its diagonal $\mathrm{diag}(\bm{J}_F)\in\mathbb{R}^M$, reducing the \emph{work} $\mathcal{W}$ of each Newton iteration from $\mathcal{W}_{\mathrm{full}}=\mathcal{O}(M^3T)$ to $\mathcal{W}_{\mathrm{diag}}=\mathcal{O}(MT)$.

While DEER and its variants can be directly used to train sequence models on data from nonlinear DS, the methods suffer from Jacobian divergence when tasked to generate orbits from systems with (on average) unstable dynamics, i.e. where $\lVert\bm{J}_{F}(\bm{z}_t) \rVert_2 > 1$ for most $t$ \citep{gonzalez2025predictability}. This divergence is inevitable when training on chaotic systems \cite{mikhaeil_difficulty_2022}. Since most if not all natural complex DS are chaotic \citep{durstewitz_dynamical_2007, FieldKorosNoyes1972, GovindanECGchaos, Sivakumar2004ChaosGeophysics, Turchin92Ecology}, this needs to be amended.

\section{Methods}
\subsection{DSR from long sequences}
To enable stable training on long sequences of chaotic dynamics, we propose to combine GTF and DEER, which we coin GTF-DEER. By simply replacing the RNN map $F_{\bm{\theta}}$ with the teacher-forced variant \eqref{eq:GTF}, the residual becomes the vector with entries
\begin{equation}\label{eq:GTF_residual}
    \left[\bm{r}_\alpha(\bm{z}_{1:T}, \bm{x}_{1:T})\right]_t = \bm{z}_t - F\left(\delta_{\alpha,\bm{B}}(\bm{z}_{t-1}, \bm{x}_{t-1})\right)
\end{equation}
for $t = 2\dots T$ and $\left[\bm{r}_\alpha(\bm{z}_{1:T}, \bm{x}_{1:T})\right]_1 = \bm{z}_1-F_{\bm{\theta}}(\bm{z}_0)$ since $\bm{z}_0$ is given.
Hence the linear recursion uses the decomposed Jacobians \eqref{eq:error_forcing_decomp}
\begin{equation}\label{eq:GTF_DEER_update}
    \Delta \bm{z}_t^{(i+1)} \\
    = \left[\bm{J}_F(\tilde{\bm{z}}_{t-1}^{(i)})\bm{P}_\alpha\right] \ \Delta\bm{z}_{t-1}^{(i+1)} - \left[\bm{r}_\alpha\left(\bm{z}_{1:T}^{(i)}, \bm{x}_{1:T}\right)\right]_t
\end{equation}
Similar to sequential GTF, Jacobian divergence caused by chaotic dynamics is tamed by $\bm{P}_\alpha$ which for $M \leq N$ strictly turns a DS with diverging state space directions into a contracting one, improving DEER convergence \citep{gonzalez2025predictability}. To demonstrate this formally, we need the following definitions.
\begin{definition}\label{def:predictability}
The \emph{largest Lyapunov exponent} (LLE) of a trajectory of the iterated map $F$ starting at $\bm{z}_0$ is
\begin{equation}
    \lambda(\bm{z}_{0})=\lim_{T \to \infty} \frac{1}{T}
\log \left\lVert{\bm{J}_{T}\bm{J}_{T-1}\dots\bm{J}_1}\right\rVert_2,
\end{equation}
where $\bm{J}_t := \bm{J}(\bm{z}_{t-1}):=\frac{\partial F(\bm{z}_{t-1})}{\partial \bm{z}_{t-1}}$.

We say the iterated map $F$ is \emph{divergent} in $\bm{z}_0$ if $\lambda(\bm{z}_{0})>0$. When $\lambda(\bm{z}_{0}) <0$, $F$ is \emph{contracting}.
\end{definition}

\begin{proposition}\label{prop:gtf_deer_convergence}
Let an SSM be given by a latent model $F: \mathbb{R}^M \to \mathbb{R}^M$ and a linear observation model $G: \mathbb{R}^M \to \mathbb{R}^N, \ \bm{z}\mapsto\bm{B}\bm{z}$ where $M \leq N$ and $\bm{B}$ is assumed to have full rank. Let the
supremal Jacobian norm satisfy $\tilde{\sigma}_{\max} := \sup \left\{\lVert\bm{J}(\bm{z})\rVert_2 = \sigma_{\mathrm{max}}(\bm{J}(\bm{z})) \mid \bm{z} \in \mathbb{R}^M \right\} > 1$.
Choose the GTF forcing strength
$\alpha \in (\alpha^*,\; 1]$, where $\alpha^* := 1 - \tilde{\sigma}_{\max}^{-1}$ \citep{hess_generalized_2023},

and define the effective contraction rate
\begin{equation}\label{eq:rho}
    \rho := (1 - \alpha)\, \tilde{\sigma}_{\max} < 1.
\end{equation}
Then the forced system $F\circ\delta_{\alpha, \bm{B}}$ is globally contracting with rate~$\rho$.
    Its LLE satisfies
    \begin{equation}\label{eq:lle_bound}
        \lambda_{\mathrm{GTF}} \leq \log \rho < 0.
    \end{equation}
\end{proposition}

\begin{proof}
See Appx. \ref{appx:gtf_deer_proof}.
\end{proof}

Proposition \ref{prop:gtf_deer_convergence} says that given a suitable $\alpha$, training a nonlinear sequence model using GTF-DEER leads to guaranteed convergence of the forward pass regardless of the dynamics that underlie the data, enabling stable and efficient parallel-in-time training for DSR. We validate this theoretical finding in Sect. \ref{subsec:gtf_deer_runtime_convergence} and investigate the effect of $M > N$ on GTF-DEER convergence empirically. Moreover, since we directly want to match forced model trajectories to targets given by the data, we can initialize the Newton iterations of DEER using the forcing signals, i.e. $\bm{z}_{1:T}^{(0)} = \overline{\bm{z}}_{1:T}$, which speeds up convergence over the typical initialization by zeros \citep{gonzalez2024towards, lim2024parallelizing}. We provide the training routine using GTF-DEER in Alg.~\ref{alg:gtf-deer}.

\subsection{Capturing long-term dependencies}\label{subsec:longterm_dependencies}
\paragraph{Regularization.} To effectively capture long-term dependencies, sequence models must maintain stable flow of information over extended temporal horizons \citep{bengio_learning_1994, gu_efficiently_2022, hochreiter_lstm_97, zucchet2024recurrent}. Crucially, for maintaining connections among states across time we must not only avoid chaotic divergence, but also prevent error signals propagated backward from decaying too fast. In this work we therefore make use of Manifold Attractor Regularization (MAR; \citep{schmidt2021identifying}) to ameliorate diminishing gradients by equipping the sequence model with a latent subspace that adaptively captures slow timescales in the data. 
In Appx.~\ref{appx:regularizations}, we review the MAR contribution to the loss function and show how this regularization leads to stable error propagation by imposing a block-diagonal identity structure in the model Jacobian $\bm{J}_F$.

\paragraph{Latent state warm-up.}
To provide the RNN with sufficient dynamical context during training, we employ latent state warm-up \citep{schiller2026tuning, vlachas_backpropagation_2020}. During training we sample sequences of length $T$, where the first $T_w$ steps are used for warm-up. The warm-up consists of forcing the sequence model with data $\bm{x}_{1:T_w}$, either by feeding them through $\bm{U}$ for the LSSM (Eq.~\eqref{eq:shplrnn_ssm}) or by GTF in the case of the nonlinear RNN (Eqs. \eqref{eq:GTF} and \eqref{eq:error_forcing}) with $\alpha=1$. This yields a latent state $\bm{z}_{T_w}$ which holds a compressed history of the warm-up signal, which is then used to predict the remaining $T-T_w$ time steps that contribute to the loss function given by the mean-squared-error $\textrm{MSE}(\bm{x}_{T_w+1:T}, \hat{\bm{x}}_{T_w+1:T})$. 

Crucially, the forward pass of the entire sequence of length $T$ \emph{including} the warm-up can be computed by a single call to GTF-DEER, enabling long contexts during training (cf. Alg.~\ref{alg:gtf-deer}).

\paragraph{Performance measures.} To evaluate DSR performance we use an established long-term measure which computes the KL divergence between multivariate state distributions in observation space, $D_\mathrm{stsp}$ \citep{hess_generalized_2023, koppe_identifying_2019}. For partially observed systems, we first perform a delay embedding \citep{sauer_embedology_1991, takens_detecting_1981} of ground-truth and generated trajectories and compute $D_\mathrm{stsp}$ on these embeddings. We will make delay embeddings explicit in notation by writing the measure as $D_\mathrm{stsp}^{\mathrm{DE}}$.
We also report the $128$-step root-mean-square-error (RMSE) as a measure of short- to medium-length prediction accuracy. For mathematical and experimental details on evaluation measures, see Appx. \ref{appx:measures}. 

\section{Results}
We will first evaluate the computational efficiency of GTF-DEER by performing runtime analyses. We will then investigate the benefit of long-sequence training for DSR through a set of ablation experiments and comparisons to state-of-the-art sequence models. Specifically, we 1) ask whether long sequences improve DSR and 2) reveal and explain the shortcomings of SSMs for DSR.

\subsection{Runtime convergence and performance of GTF-DEER}\label{subsec:gtf_deer_runtime_convergence}

\begin{figure}
    \centering
    \includegraphics[width=1.0\linewidth]{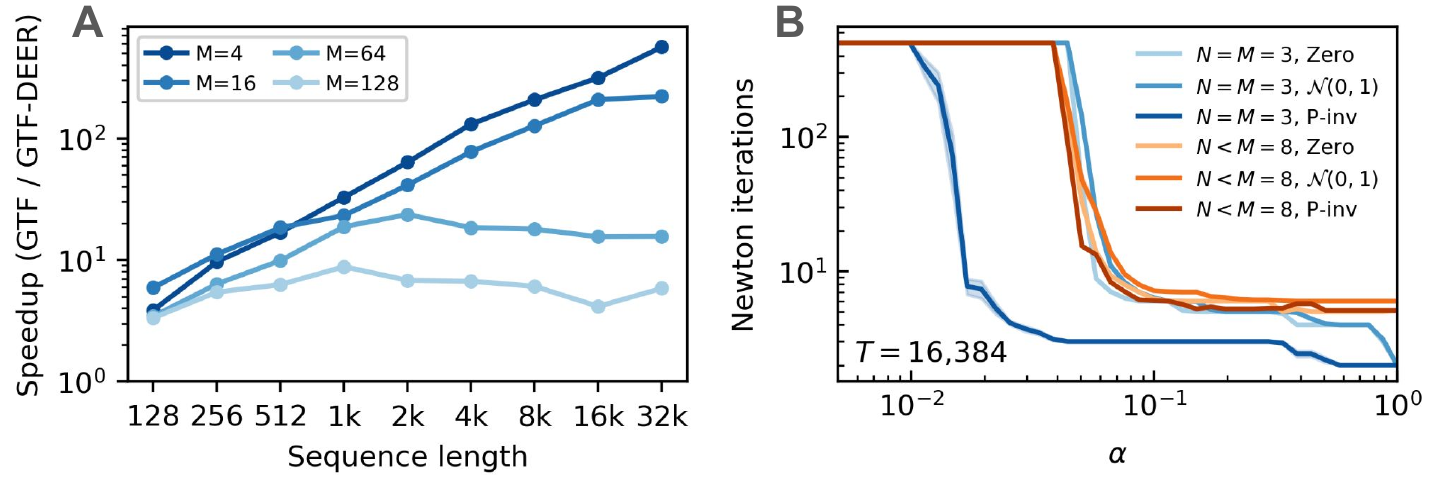}
    \caption{\textbf{A}: GTF-DEER scales favorably in sequence length, but the GPU quickly saturates for larger problems. Note the logarithmic scaling of the y-axis. \textbf{B}: The forcing parameter $\alpha$ controls Jacobian norms and hence reduces the number of Newton iterations needed for convergence of the GTF-DEER forward pass. Initializing the model trajectory using estimated latents through $\hat{\bm{z}}_{t} = \bm{B}^+ \bm{x}_{t}$ (`P-inv') improves GTF-DEER convergence over naive initializations for $N=M$; however, the effect vanishes for $M > N$. All analyses were performed on an NVIDIA RTX PRO 6000 Blackwell (96GB) GPU.}
    \label{fig:1}
\end{figure}

We evaluate the efficacy of GTF-DEER by performing runtime comparisons to the sequential baseline as well as performing ablations on several parameters of the GTF-DEER algorithm. Figure \ref{fig:1}\textbf{A} shows median runtimes for combined forward+backward passes for both sequential and parallel evaluation using GTF-DEER while training to reconstruct dynamics of the chaotic Lorenz-63 attractor ($N=3$). While the sequential approach exhibits linear scaling throughout, GTF-DEER enables sublinear scaling with sequence length, leading to a speedup of up to $\mathbf{870\times}$ over sequential evaluation (see Fig. \ref{fig:1}\textbf{A}, setting $M=4$, $T=32{,}768$). However, for large problems, i.e. large $B \times T \times M$, GPU memory bandwidth quickly saturates leading to near linear to constant scaling, a common bottleneck in parallel scan implementations on GPU devices \citep{gu2024mamba, merrill2016single}. Nevertheless, for all settings GTF-DEER outperforms naive sequential training in terms of raw runtime.
Moreover, single-device saturation does not pose a fundamental limitation: since the scan is associative, it can be chunked across $D$ devices using hierarchical scans \citep{blelloch1990prefix, dao2024transformers, Sengupta2008}, reducing work and memory per device by a factor of roughly $D$. We show in Fig.~\ref{fig:multi_gpu_speedup} that this yields further speedups with (near-)linear scaling in the number of devices $D$ on sequences of length $T=10^6$, far exceeding what could be effectively parallelized on a single device; see Appx.~\ref{appx:multi-gpu} for details.

Figure \ref{fig:1}\textbf{B} shows the convergence of the DEER forward pass under increasing values of $\alpha \in [0, 1]$. For GTF-DEER, $\alpha$ does not only control Jacobian norms during the backward pass, addressing the exploding gradient problem in the face of chaos \citep{hess_generalized_2023, mikhaeil_difficulty_2022}, but also controls convergence of the linear recurrence in Eq.~\eqref{eq:GTF_DEER_update}. If $\alpha$ is too small, the Jacobians will diverge in chaotic settings, and hence the forward pass becomes numerically unstable such that the algorithm suffers its worst-case scaling where the number of Newton iterations needed to converge
approaches the length of the sequence itself, $\mathcal{O}(T)$ \citep{gonzalez2024towards}. This defeats DEER's purpose, as the algorithm then suffers from a runtime complexity of $\mathcal{O}(T \log T)$, which is worse than generating the forward pass sequentially in $\mathcal{O}(T)$ time. Increasing $\alpha$ pushes Jacobian norms below $1$ such that the forced model constitutes a stable DS, leading to convergence of GTF-DEER in only $\mathcal{O}(\log T)$ Newton iterations, which guarantees the sublinear average-case runtime of $\mathcal{O}((\log T)^2)$ \citep{gonzalez2025predictability}. In fact, for sufficiently large $\alpha$, the forward pass even converges optimally in just $2$ Newton iterations for $M \leq N$!\footnote{The reason that two Newton iterations are required instead of just one is that an additional iteration is needed to verify that the previous iteration has indeed converged.}
For $M > N$, only a subspace of the latent space of the RNN is forced such that GTF-DEER needs more iterations to converge to the true latent dynamics. For guidance on how to choose $\alpha$ in practice, see Appx. \ref{appx:alpha_choice}.

Finally, the different curves in Fig. \ref{fig:1}\textbf{B} show the benefit of readily available forcing targets: Initializing the latent trajectory with estimated forcing signals by pseudo-inversion of the observation model drastically reduces the number of Newton iterations for a given $\alpha$ and $M \leq N$. This specific benefit vanishes when moving to the overdetermined case of $M > N$. This is expected; the RNN dynamics will veer off the estimated forcing signals as GTF is only applied to the zero-error manifold \citep{sagtekin2025error}. For more details on the experimental setup underlying Fig. \ref{fig:1} we refer to Appx. \ref{appx:details_fig1}.

\subsection{GTF-DEER ablations}\label{subsec:gtf_deer_ablations}
An important observation is that using diagonal approximation of the Jacobians in Eq.~\eqref{eq:GTF_DEER_update} for the forward pass (quasi-DEER, \citep{gonzalez2024towards}) affects gradients during the backward pass, as the same diagonal Jacobians will be used (see Appx.~\ref{appx:backward_pass} for more details). In Fig. \ref{fig:qdeer_ablation}, we test how Jacobian diagonalization affects GTF-DEER convergence and reconstruction quality. We show loss curves and Newton iterations of shPLRNNs ($M=5$, $L=50$) trained on the Lorenz-63 system under two settings: In the `FO' setting, the shPLRNN is tasked to reconstruct the dynamics from the fully observed system ($N=3$), while in the more challenging `PO' setting only the $x$-component is observed ($N=1$). We then consider training with diagonalized Jacobians (`quasi') as well as full Jacobians, i.e. standard GTF-DEER, while keeping all other training-related hyperparameters the same. We find that while full-Jacobian training handled both settings with minimal Newton iterations and provided good reconstruction quality (`PO': $D^{\mathrm{DE}}_\mathrm{stsp}=(7.5\pm6.2)\cdot10^{-3}$, `FO': $D_\mathrm{stsp}=(8.7\pm3.8)\cdot10^{-3}$), diagonal approximations needed $\approx 100\times$ more Newton iterations for the forward pass to converge, failed to yield convergent loss curves in the more challenging `PO' setting, and fell behind in reconstruction performance across both settings (`PO + quasi': $D^{\mathrm{DE}}_\mathrm{stsp}=(7 \pm 7)$ with $14/20$ runs  diverging, `FO + quasi': $D_\mathrm{stsp}=(4.4\pm3.5)\cdot10^{-2}$). This result highlights the role of the Jacobians during the backward pass: Since the diagonal approximations of the Jacobians are used to compute the gradient, they lose information on temporal mixing effects. This prevents the DSR model from learning a faithful embedding of the partially observed dynamics, which is also reflected in the erratic loss curve in Fig \ref{fig:qdeer_ablation}\textbf{A}. Thus, while full Jacobians in theory increase the time demand of each Newton step, this is more than compensated for by the orders of magnitude less iterations and the more stable training.
\begin{figure}[htb!]
    \centering
    \includegraphics[width=1.0\linewidth]{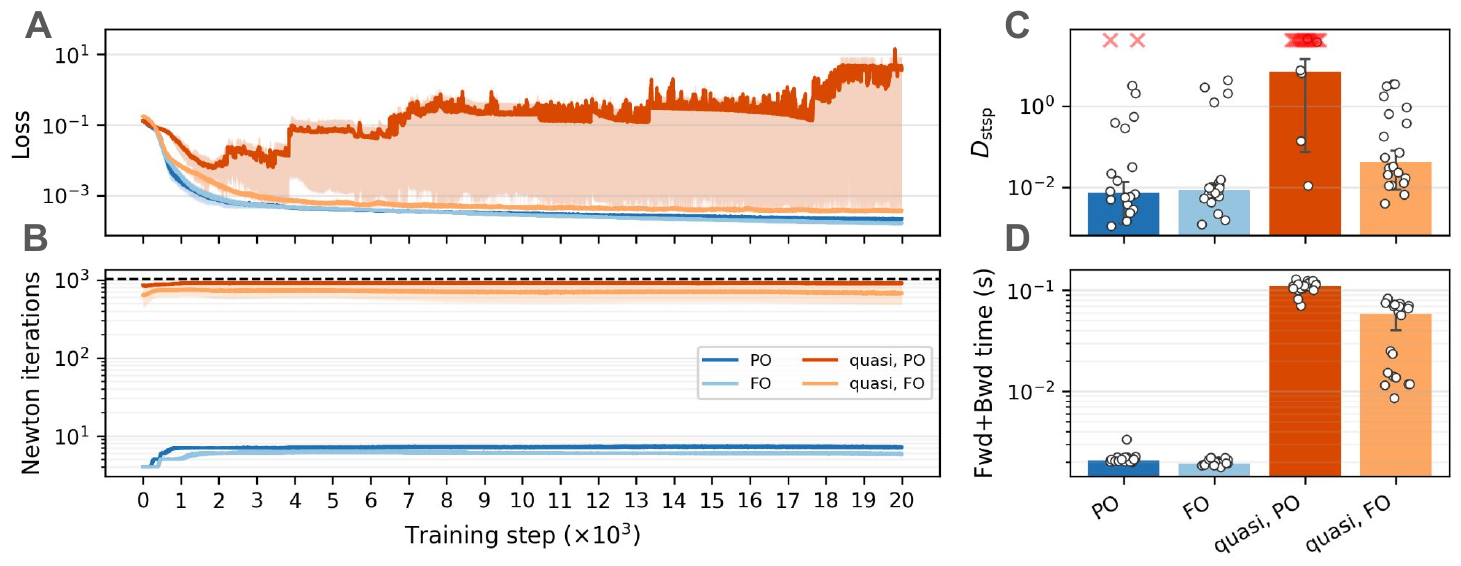}
    \caption{Training dynamics and reconstruction quality of a shPLRNN ($M=5$, $L=50$) on the Lorenz-63 system under fully observed (FO, $N=3$) and partially observed (PO, $N=1$) conditions. \textbf{A}: Training loss curves. \textbf{B}: Newton iterations per training step. \textbf{C}: State-space divergence $D_\mathrm{stsp}$ (lower is better); red crosses mark runs where trajectory divergence during generation produced NaNs. \textbf{D}: Wall-clock runtime per training step. All reported values are median $\pm$ MAD across $20$ runs.}
    \label{fig:qdeer_ablation}
\end{figure}

\subsection{Reconstructing dynamical systems with long-term dependencies}

To demonstrate the advantage of processing long sequences, we generated data from 1) a Lorenz-96 system ($N=6$) \citep{lorenz_predictability_1996} augmented with a sinusoidal forcing term of period $15{,}000$ discrete time steps, thereby introducing an explicit long time scale into the dynamics, and 2) a bursting neuron biophysical model \citep{durstewitz_dynamical_2007} which exhibits long inter-burst-intervals of $>10^4$ time steps (see Appx. \ref{appx:datasets} for further details on data generation). We trained shPLRNNs with $M=10$ for the Lorenz-96 system and $M=6$ for the bursting neuron, where both settings share $L=128$, to reconstruct the underlying dynamics while systematically varying the training sequence length. To control for the fact that longer sequences expose the model to more data per gradient update, we held the product $B \cdot T = 2^{15} = 32{,}768$ fixed. For full experimental details see Appx.~\ref{appx:fig_3a_settings}.

Figure \ref{fig:fig3}\textbf{A} reports the long-term reconstruction measure $D_{\mathrm{stsp}}^{\mathrm{DE}}$ as a function of sequence length (see Fig.~\ref{fig:rmse_fig3} for the $\textrm{RMSE}(128)$). For both systems, the measure improves with increasing sequence length, confirming that the model benefits substantially from training specifically on longer contexts, while keeping the \textit{total} amount of training data constant. Training at such sequence lengths becomes tractable only through the favorable scaling of GTF-DEER: whereas training on sequences of length $32{,}768$ required on average ${\approx}\,10$ minutes for the Lorenz-96 setting on a single NVIDIA RTX PRO 6000 Blackwell GPU, sequential GTF \citep{hess_generalized_2023} would require approximately $34$ hours.

\begin{figure}
    \centering
    \includegraphics[width=1.0\linewidth]{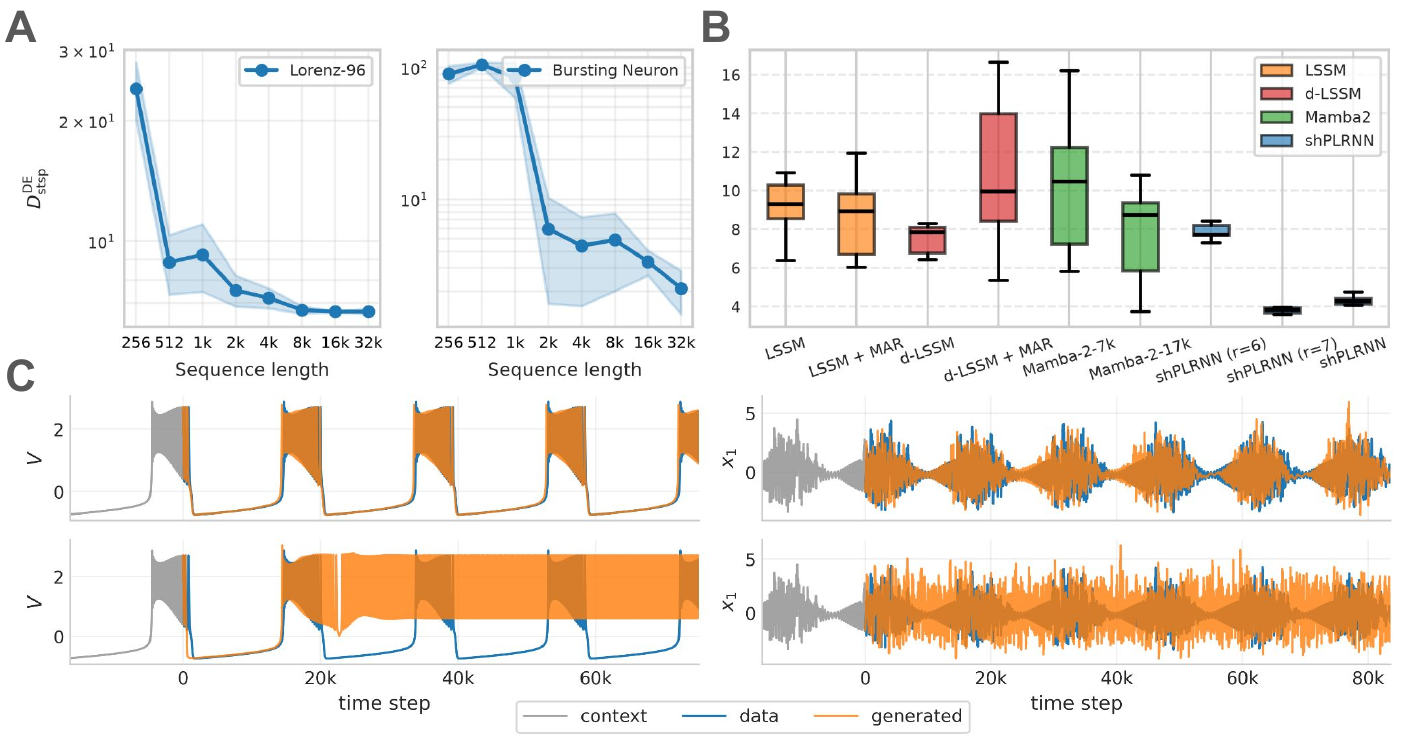}
    \caption{\textbf{A}: Long-term measure $D_{\mathrm{stsp}}^{\mathrm{DE}}$ as a function of sequence length for shPLRNNs trained on the forced Lorenz-96 and bursting neuron system. \textbf{B}: $D_{\mathrm{stsp}}^{\mathrm{DE}}$ evaluated on the forced Lorenz-96 system for different models (see Fig.~\ref{fig:model_comparison_traces} for qualitative comparison). Note that even Mamba-2 cannot match the performance of the $r=7$ shPLRNN trained with GTF-DEER, even when allowed many more parameters. \textbf{C}: Example reconstructions for long-sequence GTF-DEER (top) vs. a model trained only on standard-length sequences (bottom) for the bursting neuron (left) and forced Lorenz-96 (right).}
    \label{fig:fig3}
\end{figure}

\subsection{Linear vs. nonlinear training-time recurrences}\label{subsec:recurrence_comparison}
To further evaluate the benefits of nonlinear training-time recurrences for DSR, we compare the DSR performance of nonlinear RNNs (Eq.~\eqref{eq:shplrnn_nonlinear}) trained by GTF-DEER with that of LSSMs (Eq.~\eqref{eq:shplrnn_ssm}) and Mamba-2-based models (\citep{dao2024transformers}, Appx. \ref{appx:mamba2}) on two simulated benchmark systems (forced Lorenz-96 and the Deng--Hines model of chaotic and multi-timescale predator--prey dynamics \citep{deng2001food, hastings1991chaos}), as well as on real-world data (membrane potential recordings from cortical cells in-vivo \citep{hahn2012spontaneous}). For comparability, we fix shared hyperparameters, such as shPLRNN or MLP hidden layer size $L$, sequence length $T$ and batch size $B$. Results of all model comparisons are collected in Tab. \ref{tab:model_comparisons}.

\paragraph{Forced Lorenz-96 system.} We trained all models on the forced Lorenz-96 with $T=81{,}920$ and $B=1$, and performed several ablations: First, to investigate the effect of the low-rank constraint in the LSSM (Eq.~\eqref{eq:shplrnn_ssm_test_time}), we also train a low-rank version of the shPLRNN, where $\bm{W}:=\bm{W}_L\bm{W}_R$ with $\bm{W}_L \in \mathbb{R}^{M \times r}$ and $\bm{W}_R \in \mathbb{R}^{r \times L}$. Second, to test whether the LSSM’s shortcomings stem from the restricted diagonal parameterization of $\bm{A}$, we train it with a dense matrix $\bm{A} \in \mathbb{R}^{M \times M}$ (d-LSSM).\footnote{Note that this does not affect the rank constraint on $\overline{\bm{W}} = \bm{U}\bm{B}$, which is inherent to feeding predictions back through the input layer $\bm{U}$ (see Sect.~\ref{subsec:linear_recurrences}).} Third, to assess the effect of the MAR used for the shPLRNN, we add the same regularization term to the training of LSSMs, see Appx.~\ref{appx:regularizations} for details. Figure \ref{fig:fig3}\textbf{B} summarizes the results of these ablations quantitatively using $D^\mathrm{DE}_\mathrm{stsp}$, while qualitative example reconstructions are provided in Fig. \ref{fig:model_comparison_traces}. 

When limiting the rank to the number of observed variables, $r=N=6$, both the LSSMs and the shPLRNN fail to reconstruct the forced Lorenz-96 system and do not properly capture the influence of the sinusoidal forcing. However, increasing the rank by just $1$ to $r=7$ was already sufficient to significantly improve reconstruction quality for the shPLRNN. In contrast, LSSMs with a dense connectivity matrix do not close the gap to the shPLRNN trained with GTF-DEER.
While Mamba-2 is able to avoid the low-rank constraint through its selection mechanism, it performs much worse than the shPLRNN trained by GTF-DEER, a problem that could not be fixed by more than doubling its number of trainable parameters. Hence, the performance gap cannot be explained by model expressivity or parameterization alone, thus highlighting the importance of the training algorithm: Compared to conventional TF, GTF(-DEER) improves auto-regressive roll-outs at test time, since GTF reduces exposure bias by keeping the forcing strength $\alpha$ minimal. Applying similar strategies to address exposure bias in linear training-time recurrence models destroys their amenability to efficient parallelization (see Appx.~\ref{appx:sec:scheduled_sampling_breaks_scan}). 

\paragraph{Deng--Hines model.} Results were similar for the partially observed predator--prey system: LSSMs and the Mamba-2 models fail to reproduce the system's chaotic dynamics during autoregressive roll-outs (median $D^{\mathrm{DE}}_\mathrm{stsp} = 15$ and  $D^{\mathrm{DE}}_\mathrm{stsp} = 34.5$, respectively), while shPLRNNs trained with GTF-DEER perform well ($D^{\mathrm{DE}}_\mathrm{stsp} = 0.6$), as long as the sequence length $T$ contains enough information on the system's multiple time scales (see also Tab.~\ref{tab:model_comparisons}).

\paragraph{Neurophysiological recordings.} We also tested performance on long real-world time series, namely whole-cell patch clamp recordings of the membrane potential of single rodent entorhinal cortex neurons in-vivo \citep{hahn2012spontaneous}. These involve slow and irregular state transitions between hyperpolarized ``DOWN'' and depolarized ``UP'' states carrying fast spikes, thus multiple timescales from milliseconds to many seconds, as well as considerable noise levels, making them a challenging test case (see Fig.~\ref{fig:real_world_data_rollout}). Using GTF-DEER, we manage to train on sequences of length $T={50,000}$ (corresponding to $10$ seconds) that include multiple transitions between UP and DOWN states, yielding shPLRNNs that remain stable under autonomous roll-outs. Trained models also capture the statistical long-term behavior (median $D^{\mathrm{DE}}_\mathrm{stsp} = 0.68$), which LSSMs ($D^{\mathrm{DE}}_\mathrm{stsp} = 6.2$) and Mamba-2 ($D^{\mathrm{DE}}_\mathrm{stsp} = 50.3$) fail to reproduce. Especially Mamba-2 struggles with this real-world dataset, where roll-outs collapse to trivial fixed point dynamics reflected in large $D^{\mathrm{DE}}_\mathrm{stsp}$ values (despite similar short-term forecast performance, see Tab.~\ref{tab:model_comparisons}), confirming the exposure bias problem.

\section{Discussion}
We introduced GTF-DEER, a parallel-in-time training algorithm that combines the long-sequence scalability of DEER with the chaos-taming properties of Generalized Teacher Forcing (GTF). Through forcing, GTF-DEER turns the RNN during training into a stable DS, improving empirical runtime scaling from the worst-case of $\mathcal{O}(T\log T)$ under unstable dynamics to the average case of $\mathcal{O}((\log T)^2)$ for contractive RNNs \citep{gonzalez2025predictability}. Leveraging fast training on arbitrary sequence lengths, we provided what is to our knowledge the first systematic evidence that DSR benefits from training on such long sequences when the data carry slow timescales. Furthermore, we showed that feeding predictions of linear SSMs (LSSMs) with nonlinear read-out back into the recurrence during autoregressive generation imposes a low-rank constraint. Artificially enforcing this constraint in nonlinear RNNs prevents them from inferring unobserved dynamical variables, imposing the same limitation as in LSSMs. Even when this problem is ameliorated in (L)SSMs by introducing gating or selection mechanisms \cite{dao2024transformers, gu2024mamba}, DSR quality suffers due to the inherent exposure bias of conventional teacher forcing \cite{bengio2015scheduled, ranzato2016sequence}, which degrades auto-regressive roll-outs during testing. For modern SSMs, this problem cannot be resolved without trading the linear recurrence for a (locally) nonlinear one, such that parallelization by parallel scan and similar algorithms breaks down (cf. Appx. \ref{appx:sec:scheduled_sampling_breaks_scan}, \citep{blelloch1990prefix, dao2024transformers, gu2024mamba, gu_efficiently_2022, smith_simplified_2023}).

\paragraph{Limitations.} GTF-DEER inherits DEER's cubic work in the latent dimension $M$ per Newton-iteration, $\mathcal{W}_{\mathrm{full}}=\mathcal{O}(M^3T)$, and the diagonal (quasi-DEER) approximation that would alleviate this degrades gradients in partially observed settings substantially (Sect.~\ref{subsec:gtf_deer_ablations}). This imposes practical limits on the latent dimension $M$ (see Fig. \ref{fig:1}\textbf{A}). Finally, Proposition \ref{prop:gtf_deer_convergence} strictly holds for $M \leq N$, while for $M > N$ we only find empirical evidence that GTF-DEER's convergence also holds for the settings and problems considered in this work.

\begin{ack}
This work was supported by individual grants Du 354/14-1 (project no.\ 437610067, within FOR-5159) and Du 354/15-1 (project no.\ 502196519)
from the German Research Foundation (DFG).
We thank Maxim Braunsburger for his contributions to the implementation and evaluation of an earlier version of the GTF-DEER algorithm as part of his B.Sc. thesis under our supervision.
\end{ack}

\bibliographystyle{plain}
\bibliography{references}

\newcommand{\beginappendix}{%
        \setcounter{table}{0}
        \renewcommand{\thetable}{A\arabic{table}}%
        \setcounter{figure}{0}
        \renewcommand{\thefigure}{A\arabic{figure}}%
}

\clearpage
\appendix
\beginappendix

\section*{Appendix}

\section{Absence of chaos-induced exploding gradients in linear SSMs}\label{appx:sec:no_exploding_grads_for_ssms}

The gradient propagation properties of linear SSMs are well-established in prior work \citep{mikhaeil_difficulty_2022, orvieto2023resurrecting, zucchet2024recurrent} and we briefly recap them here for completeness. Consider a discrete-time linear SSM with time-invariant dynamics:
\begin{equation}
    \bm{z}_t = \bm{A} \bm{z}_{t-1} + \bm{v}_t,
\end{equation}
where $\bm{z}_t \in \mathbb{R}^M$ is the latent state, $\bm{A} \in \mathbb{R}^{M \times M}$ is the state transition matrix, and $\bm{v}_t \in \mathbb{R}^M$ is an extrinsic input. Differentiating the recurrence yields a constant Jacobian $\bm{J}_i = \partial \bm{z}_i / \partial \bm{z}_{i-1} = \bm{A}$, so the BPTT chain-rule product collapses to
\begin{equation}
    \frac{\partial \bm{z}_t}{\partial \bm{z}_r} = \prod_{i=r+1}^{t} \bm{J}_i = \bm{A}^{t-r},
    \qquad
    \left\| \frac{\partial \bm{z}_t}{\partial \bm{z}_r} \right\| \le \| \bm{A} \|^{t-r}.
\end{equation}
Gradient propagation is therefore independent of the state trajectory $\{\bm{z}_i\}$ and inputs $\{\bm{v}_i\}$, and is fully determined by the spectrum of $\bm{A}$. As a consequence, linear SSMs cannot exhibit chaos-induced gradient explosion: constraining the eigenvalues of $\bm{A}$ (typically near the unit disk) is sufficient to control gradient magnitudes.

\section{Proof of Proposition~\ref{prop:gtf_deer_convergence}: GTF-DEER convergence}\label{appx:gtf_deer_proof}
\begin{proof}

Since $M \leq N$ and the observation matrix $\bm{B}$ has full column rank,
$\bm{B}^+ \bm{B} = I_M$, and by Eq.~\eqref{eq:error_forcing}, $\bm{P}_\alpha = (1-\alpha)\, I_M$.
The Jacobian decomposition simplifies to Eq.~\eqref{eq:jacobian_decomp},
\begin{equation}\label{eq:jac_vanilla}
     \tilde{\bm{J}}_{t} := \tilde{\bm{J}}(\bm{z}_{t-1}) := \bm{J}_{F \circ \delta_\alpha}(\bm{z}_{t-1})
    = (1-\alpha)\, \bm{J}_F(\tilde{\bm{z}}_{t-1}).
\end{equation}
Taking spectral norms leads to:
\begin{equation}\label{eq:contraction_step}
    (1-\alpha) \norm{\bm{J}_F(\tilde{\bm{z}}_{t-1})}_2
    \leq (1-\alpha)\, \tilde{\sigma}_{\max}
    = \rho.
\end{equation}
Since $\alpha > \alpha^* = 1 - 1/\tilde{\sigma}_{\mathrm{max}}$, we have
$(1-\alpha) < 1/\tilde{\sigma}_{\mathrm{max}}$ and hence $\rho < 1$.

For any product of $k$ consecutive forced Jacobians:
\begin{equation}\label{eq:product_decay}
    \norm{\tilde{\bm{J}}_{t} \tilde{\bm{J}}_{t-1}\cdots\tilde{\bm{J}}_{t-k+1}}_2 \leq \prod_{i=0}^{k-1} \norm{\tilde{\bm{J}}_{t-i}}_2 \leq \rho^{k}
\end{equation}
Then by Def.~\ref{def:predictability}, we have
\begin{equation}
    \lambda(\bm{z})
    = \lim_{T \to \infty} \frac{1}{T} 
\log \norm{\prod_{i=0}^{T-1} \tilde{\bm{J}}_{T-i}}_2
    \leq \lim_{T \to \infty} \frac{1}{T} \log \rho^T
    = \log \rho < 0.
\end{equation}
Thus the forced system is \emph{contracting}.
\end{proof}
\section{Scheduled sampling breaks the linearity of LSSMs}
\label{appx:sec:scheduled_sampling_breaks_scan}
Conventional teacher forcing \citep{Goodfellow-et-al-2016} feeds the ground-truth observation $\bm{x}_{t-1}$ into the latent recurrence
\begin{equation}\label{eq:appx:ss_recurrence}
    \bm{z}_t = \bm{A}\bm{z}_{t-1} + \bm{U}\bm{x}_{t-1}
               + \bm{C}\bm{s}_t + \bm{h}
\end{equation}
at every step. While the following holds for any forcing strategy that re-introduces a feedback from model-predicted observations $\hat{\bm{x}}$ into the latent recurrence, we will exemplify the problem using the well established method of scheduled sampling \citep{bengio2015scheduled}. Scheduled sampling switches between teacher
forcing and free-running generation during training by replacing $\bm{x}_{t-1}$ with the model's own prediction $\hat{\bm{x}}_{t-1}$ with a scheduled probability
$\epsilon_t \in [0,1]$ ,
\begin{equation}
    \tilde{\bm{x}}_{t-1} =
    \begin{cases}
        \bm{x}_{t-1}        & \text{with probability } 1 - \epsilon_t, \\
        \hat{\bm{x}}_{t-1}  & \text{with probability } \epsilon_t,
    \end{cases}
\end{equation}
so that the latent recurrence used during training reads
\begin{equation}\label{eq:appx:ss_recurrence}
    \bm{z}_t = \bm{A}\bm{z}_{t-1} + \bm{U}\tilde{\bm{x}}_{t-1}
               + \bm{C}\bm{s}_t + \bm{h}.
\end{equation}
Whenever the prediction branch is taken, substituting
$\hat{\bm{x}}_{t-1} = \bm{B}\phi(\bm{V}\bm{z}_{t-1} + \bm{b})$ from
Eq.~\eqref{eq:shplrnn_ssm} into Eq.~\eqref{eq:appx:ss_recurrence} yields
\begin{equation}\label{eq:appx:ss_nonlinear}
    \bm{z}_t = \bm{A}\bm{z}_{t-1}
               + \overline{\bm{W}}\,\phi(\bm{V}\bm{z}_{t-1} + \bm{b})
               + \bm{C}\bm{s}_t + \bm{h},
\end{equation}
with $\overline{\bm{W}} = \bm{U}\bm{B}$, i.e.\ exactly the test-time recurrence of
Eq.~\eqref{eq:shplrnn_ssm_test_time}. The transition
$\bm{z}_{t-1}\mapsto\bm{z}_t$ is therefore no longer linear (affine) in $\bm{z}_{t-1}$ but
contains the nonlinearity $\phi$.

\paragraph{Incompatibility with parallel scan}
Evaluation of Eq.~\eqref{eq:shplrnn_ssm} via parallel scan
\citep{blelloch1990prefix, martin2018parallelizing, smith_simplified_2023} hinges on each transition being representable as an affine map $\bm{z}_t = \bm{A}_t\bm{z}_{t-1} + \bm{b}_t$,
identified with the pair $(\bm{A}_t, \bm{b}_t)$. Composition of two such maps,
\begin{equation}\label{eq:affine_composition}
    (\bm{A}_2, \bm{b}_2) \circ (\bm{A}_1, \bm{b}_1)
    = \bigl(\bm{A}_2\bm{A}_1,\; \bm{A}_2\bm{b}_1 + \bm{b}_2\bigr),
\end{equation}
is associative and again affine, such that partial compositions over disjoint chunks of the training sequence can be combined in any order, enabling an $\mathcal{O}(\log T)$-depth
reduction. In the standard LSSM recurrence the input term $\bm{U}\bm{x}_{t-1}$
depends only on observed data and contributes solely to $\bm{b}_t$, leaving the map affine in the latent state.

Once $\bm{x}_{t-1}$ is replaced by $\hat{\bm{x}}_{t-1}$, the input becomes a nonlinear function of $\bm{z}_{t-1}$ and the transition is no longer of this affine form. The composition of two such nonlinear transitions cannot in general be represented by a fixed, finite-parameter operator that is \textit{independent} of the latent state, such that the recurrence cannot be reduced to an associative scan. Training must instead unroll Eq.~\eqref{eq:appx:ss_nonlinear} sequentially, bringing it back to the $\mathcal{O}(T)$ bottleneck.

\section{Scaling to multiple GPUs}\label{appx:multi-gpu}
As discussed in Sect.~\ref{subsec:gtf_deer_runtime_convergence}, processing very long sequences can quickly saturate a single device because each Newton iteration employs a resource-intensive parallel scan. Assuming $P$ effective processors on one GPU, its runtime can be bounded as
\begin{equation}
    R_1(T)
    \in
    \mathcal{O}\left(\frac{T}{P}+\log T\right).
\end{equation}
With sufficient parallel resources, $P \gtrsim T$, the work term becomes negligible and the theoretical logarithmic depth $\mathcal{O}(\log T)$ is recovered. In practice, however, $T > P$ requires the device to execute the scan operations in multiple waves (sequentially), such that the $T/P$ term can dominate.
To alleviate this performance ceiling, we can distribute the scan across $D>1$ devices using a hierarchical or chunked-scan strategy \citep{blelloch1990prefix, dao2024transformers, Sengupta2008}. The sequence is partitioned into $D$ contiguous chunks of length $T_{\mathrm{sub}} =\left\lceil\frac{T}{D}\right\rceil$.
Each device first computes the prefix scan within its local chunk. The resulting $D$ chunk-level transformations (see right-hand side of Eq. \eqref{eq:affine_composition}) are then exchanged between devices and combined using an associative scan of depth \(\mathcal{O}(\log D)\). Finally, the appropriate prefix transformation is applied to each local chunk. Again assuming $P$ effective processors per device, the resulting runtime is
\begin{equation}
    R_D(T)
    \in
    \mathcal{O}\left(
        \frac{T}{DP}
        + \log T_{\mathrm{sub}}
        + \log D
        + C_{\mathrm{comm}}(D)
    \right),
\end{equation}
where $C_{\mathrm{comm}}(D)$ accounts for communication overhead such as exchanging the chunk-level transformations and synchronizing the devices. Since $\log T_{\mathrm{sub}}+\log D \simeq \log T$, this simplifies to
\begin{equation}
    R_D(T)
    \in
    \mathcal{O}\left(
        \frac{T}{DP}
        + \log T
        + C_{\mathrm{comm}}(D)
    \right),
\end{equation}
Thus, distributing the scan reduces the work and memory per device by approximately a factor of $D$. This in turn leads to (near-)linear speedups when the work $T/(DP)$ dominates both the logarithmic depth and communication overhead.

To demonstrate this speedup, we trained a shPLRNN ($M=20, L=128$) on the Lorenz-63 system with a sequence length of $T=1{,}000{,}000$ and distributed the work over 1, 2 or 4 NVIDIA RTX PRO 6000 Blackwell (96GB) GPUs while measuring summed forward and backward pass times. The forcing $\alpha$ was fixed at $\alpha =0.15$ and all settings were compared to the sequential baseline. The results are summarized in Fig.~\ref{fig:multi_gpu_speedup}.

\begin{figure}
    \centering
    \includegraphics[width=0.5\linewidth]{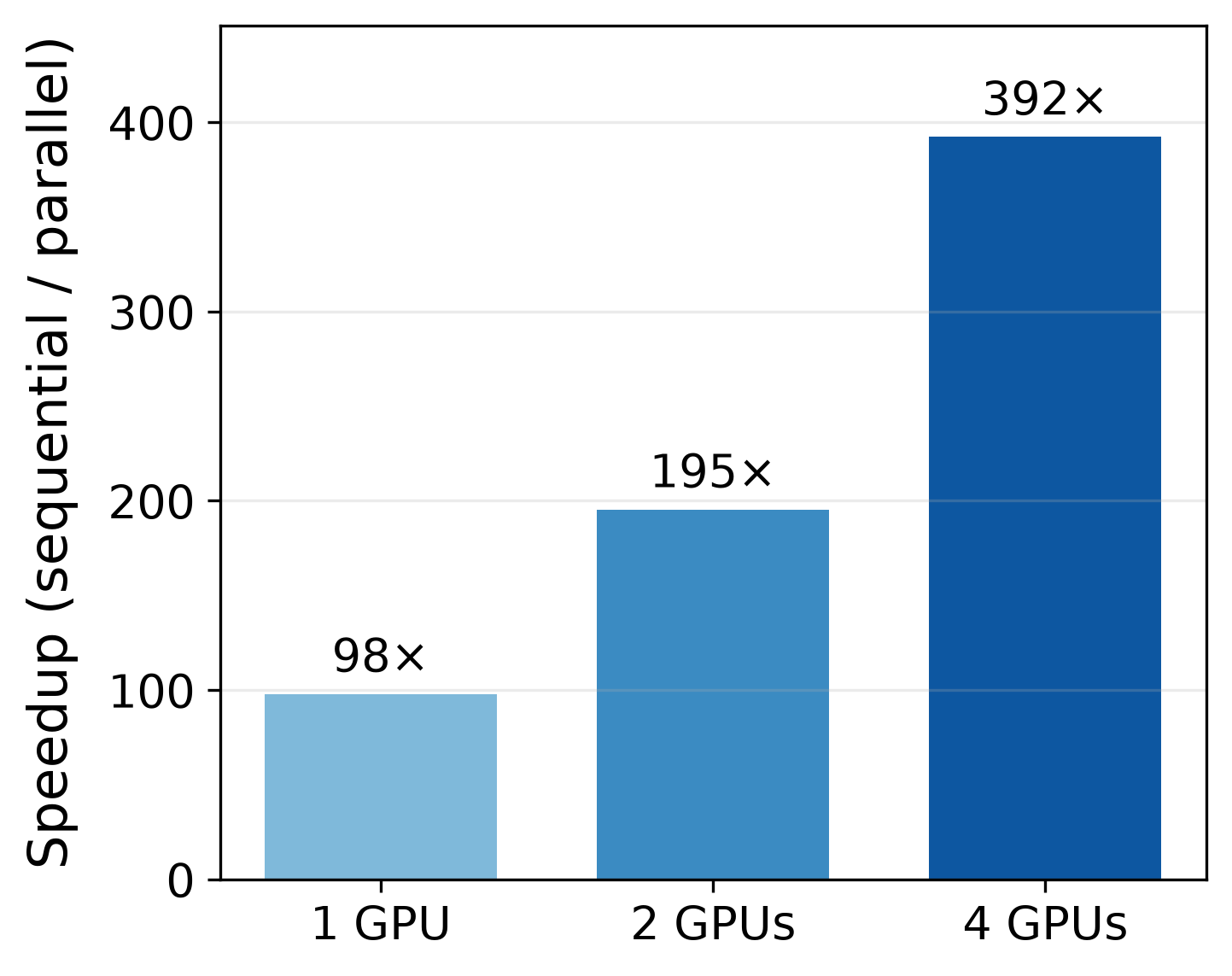}
    \caption{Speedup of distributing \emph{very long} sequences across multiple devices. A shPLRNN ($M=20, L=128$) was trained on a sequence of length $T=10^6=1{,}000{,}000$ using GTF-DEER. In this example, distributing the work to four devices leads to a $4\times$ speedup over using a single device, reducing combined average forward solve and gradient computation times from $209$ to $52$ milliseconds.}
    \label{fig:multi_gpu_speedup}
\end{figure}

\section{Gradients}\label{appx:backward_pass}

\subsection{Derivation of GTF-DEER gradient formulas}

We define the per-timestep MSE loss function (excluding regularization terms) $\ell : \mathbb{R}^N \times \mathbb{R}^N \to \mathbb{R}$,
\begin{equation}\label{eq:mse_loss_one_step}
    \ell(\bm{x}, \hat{\bm{x}}) = \frac{1}{N} \norm{\bm{x} - \hat{\bm{x}}}_2^2,
\end{equation}
and the total loss, $\mathcal{L} : (\mathbb{R}^N)^T \times (\mathbb{R}^N)^T \to \mathbb{R}$, for the full time series,
\begin{equation}\label{eq:mse_loss_series}
    \mathcal{L}(\bm{x}_{1:T}, \hat{\bm{x}}_{1:T}) = \frac{1}{T} \sum_{t=1}^T \ell(\bm{x}_t, \hat{\bm{x}}_t) = \frac{1}{NT} \sum_{t=1}^T \norm{\bm{x}_t - \hat{\bm{x}}_t}_2^2,
\end{equation}
where in the model \eqref{eq:general_stsp_model}, $\bm{\hat{x}}_t = \bm{G}_\psi(\bm{z}_t)$ are the observations associated with the latent states; we consider $\bm{x}_{1:T} \in \mathbb{R}^{NT} \cong (\mathbb{R}^N)^T$ as a stacked vector; and in practice truncate a warm-up period $0 < T_w < T$ and only supply vectors $\bm{x}, \hat{\bm{x}}$ of size $N\tilde{T}$, $\tilde{T} \coloneqq T-T_w$, cf.\ Sect.~\ref{subsec:longterm_dependencies} and Eq.~\eqref{eq:mse_loss_warmup}.

To optimize the parameters $\bm{\theta}$ of the RNN $\tilde{F}_{\bm{\theta},\alpha} = F_{\bm{\theta}} \circ \delta_{\alpha, \bm{B}}$, we need to compute the gradient of the loss function,
\begin{equation}
    \frac{\partial\mathcal{L}}{\partial\theta} = \frac{\partial\mathcal{L}}{\partial\bm{z}} \frac{\partial\bm{z}}{\partial\theta}
\end{equation}
for all elements $\theta$ (for notational simplicity, we choose not to write any index) in the abstract parameter vector $\bm{\theta}$ which comprises all matrix weights.

In the DEER forward pass, we use the residual
\begin{gather}\label{eq:residual}
    \bm{r}_{\bm{\theta}, \alpha}(\bm{z}_{t-1}, \bm{z}_{t}) = \bm{z}_t - \tilde{F}_{\bm{\theta},\alpha}(\bm{z}_{t-1}, \bm{x}_{t-1}) \in \mathbb{R}^M,\\
    \bm{R}_{\bm{\theta},\alpha}(\bm{z}_{1:T}) = (\bm{r}_\alpha(\bm{z}_0, \bm{z}_1), \dots, \bm{r}_\alpha(\bm{z}_{T-1}, \bm{z}_T)) \in \mathbb{R}^{MT},
\end{gather}
where $\bm{z}_0$ is a fixed initial state (which is not predicted), to obtain the predicted time series $\bm{z}_{1:T} \in \mathbb{R}^{MT} \cong (\mathbb{R}^M)^T$.
At the end of each forward pass,
\begin{equation}
    \bm{R}_\alpha(\bm{z}_{1:T},\bm{\theta}) \coloneqq \bm{R}_{\bm{\theta}, \alpha}(\bm{z}_{1:T}) = \bm{0} \in \mathbb{R}^{MT},
\end{equation}
so by the implicit function theorem,
\begin{equation}
    \frac{\partial\bm{R}_\alpha}{\partial\bm{z}_{1:T}} \frac{\partial\bm{z}_{1:T}}{\partial\theta}  + \frac{\partial\bm{R}_\alpha}{\partial\theta} = \bm{0} \in \mathbb{R}^{MT}.
\end{equation}
Rearranging this equation gives us
\begin{equation}
    \frac{\partial\bm{z}_{1:T}}{\partial\theta} = - \left(\frac{\partial\bm{R}_\alpha}{\partial\bm{z}_{1:T}}\right)^{-1} \frac{\partial\bm{R}_\alpha}{\partial\theta}.
\end{equation}
We insert this into the gradient of the loss:
\begin{equation}
    \frac{\partial\mathcal{L}}{\partial\theta} = \underbrace{- \frac{\partial\mathcal{L}}{\partial\bm{z}_{1:T}} \left(\frac{\partial\bm{R}_\alpha}{\partial\bm{z}_{1:T}}\right)^{-1}}_{\eqqcolon \bm{\mathcal{V}}^\top \in \mathbb{R}^{1\times MT}} \frac{\partial\bm{R}_\alpha}{\partial\theta},
\end{equation}
where we write $\bm{\mathcal{V}} = (\bm{v}_1, \dots, \bm{v}_T) \in \mathbb{R}^{MT} \cong (\mathbb{R}^M)^T$.
Rather than explicitly inverting the huge Jacobian matrix $\frac{\partial\bm{R}_\alpha}{\partial\bm{z}_{1:T}}$, we instead solve the adjoint problem
\begin{equation}\label{eq:adjoint_eq}
    \bm{\mathcal{V}}^\top \frac{\partial\bm{R}_\alpha}{\partial\bm{z}_{1:T}} = - \frac{\partial\mathcal{L}}{\partial\bm{z}_{1:T}}.
\end{equation}
To make this more explicit, we use 
\begin{gather}
    \frac{\partial\bm{r}_{\alpha, \bm{\theta}}(\bm{z}_{t-1}, \bm{z}_t)}{\partial\bm{z}_{t-1}} = -\frac{\partial\tilde{F}_{\bm{\theta},\alpha}}{\partial\bm{z}}(\bm{z}_{t-1}) \eqqcolon - \bm{J}_{\tilde{F}}(\bm{z}_{t-1}) \eqqcolon - \tilde{\bm{J}}_{t} \in \mathbb{R}^{M\times M}, \\
    \frac{\partial\bm{r}_{\alpha, \bm{\theta}}(\bm{z}_{t-1}, \bm{z}_t)}{\partial\bm{z}_{t}} = \bm{I}_M,
\end{gather}
and write
\begin{equation}\label{eq:blockbidiag_jacobian}
    \frac{\partial\bm{R}_\alpha}{\partial\bm{z}_{1:T}} = {\renewcommand{\arraystretch}{1.5}\begin{pmatrix}
        \bm{I}_M & 0 & \cdots & \cdots & 0 \\
        -\tilde{\bm{J}}_2 & \bm{I}_M & 0 & \cdots & 0 \\
        0 & -\tilde{\bm{J}}_3 & \bm{I}_M & \ddots & 0 \\
        \vdots & \vdots & \ddots & \ddots & 0 \\
        0 & 0 & 0 & -\tilde{\bm{J}}_{T} & \bm{I}_M
    \end{pmatrix}}
\end{equation}
and
\begin{equation}
    \frac{\partial\bm{R}_\alpha}{\partial\theta} = \frac{\partial\bm{R}_\alpha}{\partial\tilde{F}_{\bm{\theta},\alpha}} \frac{\partial\tilde{F}_{\bm{\theta},\alpha}}{\partial\theta} = - \frac{\partial\tilde{F}_{\bm{\theta},\alpha}}{\partial\theta}.
\end{equation}
The block-bidiagonal structure of the matrix turns \eqref{eq:adjoint_eq} into a (backward) recursion in $t$:
\begin{equation}\label{eq:backward_recursion}
    \bm{v}_T^\top = - \frac{\partial\mathcal{L}}{\partial\bm{z}_T} ,\qquad \bm{v}_{t-1}^\top = \bm{v}_t^\top \bm{J}_{\tilde{F}}(\bm{z}_{t-1}) - \frac{\partial\mathcal{L}}{\partial\bm{z}_{t-1}},\qquad t=T,\dots,2
\end{equation}
where, substituting in the observation model, $\hat{\bm{x}} = G_\psi(\bm{z})$,
\begin{equation}\label{eq:loss_partial_z}
    \frac{\partial\mathcal{L}}{\partial\bm{z}_{t-1}} = \frac{\partial\mathcal{L}}{\partial\hat{\bm{x}}_{t-1}} \frac{\partial G_\psi}{\partial\bm{z}_{t-1}} = \frac{2}{NT} (\hat{\bm{x}}_{t-1}-\bm{x}_{t-1})^\top \frac{\partial G_\psi}{\partial\bm{z}_{t-1}}.
\end{equation}
In the same way as in the forward pass, this recursion can be solved in parallel via an associative scan.
Unrolled in time, this is
\begin{equation}\label{eq:adjoint_unrolled}
    \bm{v}_t^\top = -\sum_{\tau=t}^T \frac{\partial \mathcal{L}}{\partial \bm{z}_\tau} \tilde{\bm{J}}_\tau \tilde{\bm{J}}_{\tau-1} \dots \tilde{\bm{J}}_{t+1},
\end{equation}
where the $\tau=t$ case is to be read as the empty product of Jacobians and hence the identity matrix, and $\tau=t+1$ as a single factor of $\tilde{\bm{J}}_{t+1}$.
Note that for quasi-DEER, the Jacobians $\bm{J}_{\tilde{F}}$ in \eqref{eq:backward_recursion} are replaced by a diagonal approximation and the equation becomes
\begin{equation}
    \bm{v}_T^\top = - \frac{\partial\mathcal{L}}{\partial\bm{z}_T} ,\qquad \bm{v}_{t-1}^\top = \left(\bm{v}_t \odot \mathrm{diag}\left(\bm{J}_{\tilde{F}}(\bm{z}_{t-1})\right)\right)^\top - \frac{\partial\mathcal{L}}{\partial\bm{z}_{t-1}}
\end{equation}
with the element-wise product $\odot$.

In summary, the loss gradients are therefore calculated simply by solving the recursion \eqref{eq:backward_recursion} using the final latent trajectory and observations computed via the forward pass, and inserting the result into
\begin{equation}\label{eq:loss_gradient_final}
    \frac{\partial\mathcal{L}}{\partial\theta} = - \bm{\mathcal{V}}^\top \frac{\partial\tilde{F}_{\bm{\theta},\alpha}}{\partial\theta}.
\end{equation}

\subsection{shPLRNN gradients}

We now apply the formulas derived in the previous subsection to the shPLRNN model, Eq.~\eqref{eq:shplrnn_nonlinear}. The forced recurrence, Eq.~\eqref{eq:GTF}, with forcing as in Eq.~\eqref{eq:error_forcing}, takes the explicit form
\begin{multline}\label{eq:forced_recurrence}
\tilde{F}_{\bm{\theta},\alpha}(\bm{z}_{t-1},\bm{x}_{t-1}) = \bm{A} (\bm{P}_\alpha \bm{z}_{t-1} + \alpha \bm{B}^+ \bm{x}_{t-1}) + \\ + \bm{W} \operatorname{ReLU}(\bm{V} (\bm{P}_\alpha \bm{z}_{t-1} + \alpha \bm{B}^+ \bm{x}_{t-1}) + \bm{b}) + \bm{C} \bm{s}_t + \bm{h}.
\end{multline}
The forced Jacobian is
\begin{equation}
    \bm{J}_{\tilde{F}}(\bm{z}_{t-1}) = \left(\bm{A} + \bm{W}\operatorname{diag}\left(\Theta(\bm{V}(\bm{P}_\alpha \bm{z}_{t-1} + \alpha \bm{B}^+ \bm{x}_{t-1}) + \bm{b})\right)\bm{V}\right)\bm{P}_\alpha,
\end{equation}
where the Heaviside step function $\Theta$ appears as the derivative of the ReLU function and is applied element-wise.

Inserting the observation model, $\hat{\bm{x}}_t = \bm{B}\bm{z}_t$, into Eq.~\eqref{eq:loss_partial_z} gives
\begin{equation}
    \frac{\partial\mathcal{L}}{\partial\bm{z}_t} = \frac{2}{NT} (\bm{B}\bm{z}_t - \bm{x}_t)^\top \bm{B}.
\end{equation}

Combined, the previous two formulae, inserted into Eq.~\eqref{eq:backward_recursion}, give the explicit adjoint recurrence
\begin{align}
    \bm{v}_T^\top &= -\frac{2}{NT} (\bm{B}\bm{z}_T-\bm{x}_T)^\top\bm{B} \\
    \bm{v}_{t-1}^\top &= \Big[\bm{v}_t^\top \bm{A}+\big(\Theta(\bm{V} (\bm{P}_\alpha \bm{z}_{t-1}+\alpha \bm{B}^+\bm{x}_{t-1})+\bm{b})^\top\odot(\bm{v}_t^\top \bm{W})\big)\bm{V}\Big]\bm{P}_\alpha \\
    &- \frac{2}{NT}(\bm{B}\bm{z}_{t-1}-\bm{x}_{t-1})^\top \bm{B},\qquad t=T,\dots,2.
\end{align}
The unrolled expression, Eq.\ \eqref{eq:adjoint_unrolled}, becomes
\begin{equation}
    \bm{v}_t^\top = -\frac{2}{NT} \sum_{\tau=t}^T (\bm{B}\bm{z}_\tau-\bm{x}_\tau)^\top \bm{B} \tilde{\bm{J}}_\tau \tilde{\bm{J}}_{\tau-1} \dots \tilde{\bm{J}}_{t+1}.
\end{equation}

Taking the parameter derivatives of Eq.\ \eqref{eq:forced_recurrence} and writing out the inner product with $\bm{\mathcal{V}} = (\bm{v}_1,\dots,\bm{v}_T)$ (cf.\ Eq.\ \eqref{eq:loss_gradient_final}) as a sum, we obtain the expressions
\begin{align}
\frac{\partial \mathcal{L}}{\partial \bm{A}} &= -\sum_{t=1}^{T} \bigl[\bm{P}_\alpha\bm{z}_{t-1}+\alpha\bm{B}^{+}\bm{x}_{t-1}\bigr]\, \bm{v}_t^{\top} \label{eq:loss_gradient_A} \\
\frac{\partial \mathcal{L}}{\partial \bm{W}} &= -\sum_{t=1}^{T} \operatorname{ReLU}\Bigl(\bm{V}\bigl[\bm{P}_\alpha\bm{z}_{t-1}+\alpha\bm{B}^{+}\bm{x}_{t-1}\bigr]+\bm{b}\Bigr)\, \bm{v}_t^{\top} \\
\frac{\partial \mathcal{L}}{\partial \bm{V}} &= -\sum_{t=1}^{T} \bigl[\bm{P}_\alpha\bm{z}_{t-1}+\alpha\bm{B}^{+}\bm{x}_{t-1}\bigr]\, \Bigl(\Theta\bigl(\bm{V}\bigl[\bm{P}_\alpha\bm{z}_{t-1}+\alpha\bm{B}^{+}\bm{x}_{t-1}\bigr]+\bm{b}\bigr)\odot(\bm{W}^{\top}\bm{v}_t)\Bigr)^{\top} \\
\frac{\partial \mathcal{L}}{\partial \bm{b}} &= -\sum_{t=1}^{T} \Bigl(\Theta\bigl(\bm{V}\bigl[\bm{P}_\alpha\bm{z}_{t-1}+\alpha\bm{B}^{+}\bm{x}_{t-1}\bigr]+\bm{b}\bigr)\odot(\bm{W}^{\top}\bm{v}_t)\Bigr)^{\top} \\
\frac{\partial \mathcal{L}}{\partial \bm{C}} &= -\sum_{t=1}^{T} \bm{s}_t\,\bm{v}_t^{\top} \\
\frac{\partial \mathcal{L}}{\partial \bm{h}} &= -\sum_{t=1}^{T} \bm{v}_t^{\top}.
\end{align}
With diagonal parameterization, $\bm{A} = \operatorname{diag}(\bm{a}),\, \bm{a} = \tanh(\bar{\bm{a}}), \, \bar{\bm{a}}\in \mathbb{R}^M$, Eq.~\eqref{eq:loss_gradient_A} becomes
\begin{equation}
    \frac{\partial \mathcal{L}}{\partial \bar{\bm{a}}} = - \sum_{t=1}^T \left( \bm{v}_t \odot \left( \bm{P}_\alpha \bm{z}_{t-1} + \alpha \bm{B}^+ \bm{x}_{t-1} \right) \odot \left(\bm{1}-\bm{a}\odot\bm{a}\right) \right)^\top.
\end{equation}
Moreover,
\begin{equation}
\begin{aligned}
\frac{\partial\mathcal{L}}{\partial\bm{B}}
={}&\frac{2}{N T}\sum_{t=1}^{T}\bm z_t(\bm B\bm z_t-\bm x_t)^\top
+\sum_{t=2}^{T}\alpha_{t-1}\,\bm z_{t-1}\bm v_t^\top\bm J_F(\bm q_{t-1})\,\bm B^+\\[2pt]
&+\sum_{t=1}^{T}\Big[\bm B^+\bm y_{t-1}\bm v_t^\top\bm J_F(\bm q_{t-1})\bm B^+
-\bm B^+(\bm B^+)^\top\bm J_F(\bm q_{t-1})^\top\bm v_t\,\bm y_{t-1}^\top(\bm I_N-\bm B\bm B^+)\\
&\qquad\;\;-(\bm I_M-\bm B^+\bm B)\,\bm J_F(\bm q_{t-1})^\top\bm v_t\,\bm y_{t-1}^\top(\bm B^+)^\top\bm B^+\Big],
\end{aligned}
\end{equation}
where we introduced $\bm q_{t-1}=\bm P_{\alpha_{t-1}}\bm z_{t-1}+\alpha_{t-1}\bm B^+\bm x_{t-1}$ for compactness.

Note that all derivatives are in numerator layout, so that $\partial\mathcal{L}/\partial\bm{W}$ has the transposed shape of $\bm{W}$, and for parameter updates such as $\bm{W} \leftarrow \bm{W} - \eta \nabla_{\bm{W}} \mathcal{L}$, we need to take the transpose, $\nabla_{\bm{W}} \mathcal{L} = \left({\partial \mathcal{L}}/{\partial \bm{W}}\right)^\top$. 

Taking warm-up into account, normalization factors would be replaced by ${2}/{N\tilde{T}}$ and loss terms would be zero for $t \leq T_w$. Moreover, the forcing $\alpha$ in recursion formulas would be $t$-dependent, $\bm{P}_{\alpha_{t-1}} \bm{z}_{t-1} + \alpha_{t-1}\bm{B}^+ \bm{x}_{t-1}$, with $\alpha_t=1$ during warmup.

In practical implementations, all of these gradients will be computed internally by automatic differentiation.

\subsection{LSSM gradients}

We now state the gradient expressions for the LSSM, Eq.~\eqref{eq:shplrnn_ssm}. The same general formulae as for the shPLRNN apply for this model class, too.
The training-time recurrence is linear in $\bm{z}$ and does not depend on the observation parameters $\bm{\psi} = \{\bm{B}, \bm{V}, \bm{b}\}$,
\begin{equation}
    F_{\bm{\theta}}(\bm{z}_{t-1}, \bm{x}_{t-1}, \bm{s}_t) = \bm{A}\bm{z}_{t-1} + \bm{U}\bm{x}_{t-1} + \bm{C}\bm{s}_t + \bm{h}, \qquad \bm{\theta} = \{\bm{A}, \bm{U}, \bm{C}, \bm{h}\},
\end{equation}
so the Jacobian is constant, $\tilde{\bm{J}} = {\bm{J}} = \bm{A}$ for all $t$. In the LSSM training protocol, no forcing operator or pseudo-inverse appears. Its forward pass is a single linear scan.

With the intermediate variables
\begin{align}
    \bm{u}_t &\coloneqq \bm{V}\bm{z}_t + \bm{b} \in \mathbb{R}^L, \\
    \bm{a}_t &\coloneqq \phi(\bm{u}_t) = \operatorname{ReLU}(\bm{V}\bm{z}_t + \bm{b}) \in \mathbb{R}^L, \\
    \hat{\bm{x}}_t &= \bm{B}\bm{a}_t = \bm{B}\operatorname{ReLU}(\bm{V}\bm{z}_t + \bm{b}) \in \mathbb{R}^N, \\
    \bm{e}_t &\coloneqq \hat{\bm{x}}_t - \bm{x}_t = \bm{B}\operatorname{ReLU}(\bm{V}\bm{z}_t + \bm{b}) - \bm{x}_t \in \mathbb{R}^N, \\
    \bm{k}_t &\coloneqq \Theta(\bm{u}_t) \odot (\bm{B}^\top \bm{e}_t) = \Theta(\bm{V}\bm{z}_t + \bm{b}) \odot \Bigl(\bm{B}^\top\bigl(\bm{B}\operatorname{ReLU}(\bm{V}\bm{z}_t + \bm{b}) - \bm{x}_t\bigr)\Bigr) \in \mathbb{R}^L,
\end{align}
with the Heaviside step function $\Theta$ acting element-wise as in the previous section. The chain rule through the observation model gives
\begin{equation}\label{eq:lssm_loss_partial_z}
    \frac{\partial\mathcal{L}}{\partial\bm{z}_t}
    = \frac{\partial\mathcal{L}}{\partial\hat{\bm{x}}_t}\frac{\partial\hat{\bm{x}}_t}{\partial\bm{a}_t}\frac{\partial\bm{a}_t}{\partial\bm{u}_t}\frac{\partial\bm{u}_t}{\partial\bm{z}_t}
    = \frac{2}{NT}\,\bm{e}_t^\top \bm{B}\operatorname{diag}\bigl(\Theta(\bm{u}_t)\bigr)\bm{V}
    = \frac{2}{NT}\,\bm{k}_t^\top \bm{V}.
\end{equation}

Inserting $\bm{J}_t = \bm{A}$ and Eq.~\eqref{eq:lssm_loss_partial_z} into Eq.~\eqref{eq:backward_recursion} gives
\begin{align}
    \bm{v}_T^\top &= -\frac{2}{NT}\,\bm{k}_T^\top \bm{V}, \\
    \bm{v}_{t-1}^\top &= \bm{v}_t^\top \bm{A} - \frac{2}{NT}\,\bm{k}_{t-1}^\top \bm{V}, \qquad t = T, \dots, 2.
\end{align}
This linear recursion can be solved in parallel by an associative scan, just like the forward pass. Unrolling the recursion yields
\begin{equation}
    \bm{v}_t^\top = -\sum_{k=t}^{T} \frac{\partial\mathcal{L}}{\partial\bm{z}_k}\,\bm{A}^{k-t}
    = -\frac{2}{NT}\sum_{k=t}^{T} \bm{k}_k^\top \bm{V}\,\bm{A}^{k-t},
\end{equation}
with $\bm{A}^0 = \bm{I}_M$, so that the sensitivity of $\bm{v}_t$ to later losses is governed by the powers of $\bm{A}$ alone (for diagonal $\bm{A}$, $\bm{A}^{k-t} = \operatorname{diag}(\bm{a})^{k-t}$, with $\bm{a}$ the diagonal of $\bm{A}$).

For $\bm{\theta} = \{\bm{A}, \bm{U}, \bm{C}, \bm{h}\}$, Eq.~\eqref{eq:loss_gradient_final} gives, with $\partial F_{\bm{\theta}}/\partial\bm{\theta}$ evaluated at $(\bm{z}_{t-1}, \bm{x}_{t-1}, \bm{s}_t)$,
\begin{align}
    \frac{\partial\mathcal{L}}{\partial\bm{A}} &= -\sum_{t=1}^{T} \bm{z}_{t-1}\bm{v}_t^\top, &
    \frac{\partial\mathcal{L}}{\partial\bm{U}} &= -\sum_{t=1}^{T} \bm{x}_{t-1}\bm{v}_t^\top, \\
    \frac{\partial\mathcal{L}}{\partial\bm{C}} &= -\sum_{t=1}^{T} \bm{s}_t\bm{v}_t^\top, &
    \frac{\partial\mathcal{L}}{\partial\bm{h}} &= -\sum_{t=1}^{T} \bm{v}_t^\top.
\end{align}
If $\bm{A} = \operatorname{diag}(\bm{a})$ is diagonal, 
\begin{equation}
    \frac{\partial \mathcal{L}}{\partial \bm{a}} = -\sum_{t=1}^{T} \left(\bm{z}_{t-1}\odot\bm{v}_t\right)^\top.
\end{equation}

The parameters $\bm{\psi} = \{\bm{B}, \bm{V}, \bm{b}\}$ do not enter the latent recurrence, so $\partial\bm{R}/\partial\bm{\psi} = \bm{0}$ and their gradients are given by the explicit dependence of $\mathcal{L}$ on them only,
\begin{align}
    \frac{\partial\mathcal{L}}{\partial\bm{B}} &= \frac{2}{NT}\sum_{t=1}^{T} \bm{a}_t\bm{e}_t^\top
    = \frac{2}{NT}\sum_{t=1}^{T} \operatorname{ReLU}(\bm{V}\bm{z}_t + \bm{b})\,\bigl(\bm{B}\operatorname{ReLU}(\bm{V}\bm{z}_t + \bm{b}) - \bm{x}_t\bigr)^\top, \\
    \frac{\partial\mathcal{L}}{\partial\bm{V}} &= \frac{2}{NT}\sum_{t=1}^{T} \bm{z}_t\bm{k}_t^\top
    = \frac{2}{NT}\sum_{t=1}^{T} \bm{z}_t\Bigl(\Theta(\bm{V}\bm{z}_t + \bm{b}) \odot \bigl(\bm{B}^\top(\bm{B}\operatorname{ReLU}(\bm{V}\bm{z}_t + \bm{b}) - \bm{x}_t)\bigr)\Bigr)^\top, \\
    \frac{\partial\mathcal{L}}{\partial\bm{b}} &= \frac{2}{NT}\sum_{t=1}^{T} \bm{k}_t^\top
    = \frac{2}{NT}\sum_{t=1}^{T} \Bigl(\Theta(\bm{V}\bm{z}_t + \bm{b}) \odot \bigl(\bm{B}^\top(\bm{B}\operatorname{ReLU}(\bm{V}\bm{z}_t + \bm{b}) - \bm{x}_t)\bigr)\Bigr)^\top.
\end{align}

As in the nonlinear case, in practice these gradients are not implemented manually, but computed via automatic differentiation.

\section{Training methods and details}\label{appx:regularizations}
\subsection{Model setup}
\paragraph{Parameterization and Initialization} For both the RNN and LSSM we parameterize $\bm{A} = \operatorname{diag}(\tanh(\bar{\bm{A}}))$ where $\bar{\bm{A}}\in \mathbb{R}^M$. The nonlinearity avoids instabilities in the $\bm{A}\bm{z}_{t-1}$ term when $\norm{\bm{A}}_2\geq1$. To facilitate capturing long time scales in the data, we initialize the RNN near the identity to exhibit long time scales from the start:
\begin{equation}
    \begin{aligned}
        \bm{\bar{A}} &= \operatorname{artanh}(\kappa) \ \mathbf{I} \\
        W_{ij} &\sim \mathcal{U}\left(-(1-\kappa) L^{-1/2}, (1-\kappa) L^{-1/2}\right) \\
        V_{ij} &\sim \mathcal{U}\left(-(1-\kappa)  M^{-1/2}, (1-\kappa) M^{-1/2}\right) \\
        C_{ij} &\sim \mathcal{U}\left(-(1-\kappa)  K^{-1/2}, (1-\kappa)K^{-1/2}\right)\\
        \bm{b} &= \bm{0} \\
        \bm{h} &= \bm{0},        
    \end{aligned}
\end{equation}
where $\mathcal{U}$ is the uniform distribution. In practice, we use $(1 - \kappa) \in [10^{-3}, 10^{-5}]$. For $M \ge N$, the observation matrix $\bm{B}$ is initialized to perform identity read-out on the first $N$ units in the RNN, i.e. $\bm{B} = \begin{bmatrix} \bm{I}_{N} & \bm{0}_{\,N \times (M-N)} \end{bmatrix}$.

For the LSSM, we follow a similar strategy:
\begin{equation}
    \begin{aligned}
        \bm{\bar{A}} &= \operatorname{artanh}(\kappa) \ \mathbf{I} \\
        U_{ij} &\sim \mathcal{U}\left(-(1-\kappa) N^{-1/2}, (1-\kappa) N^{-1/2}\right) \\
        C_{ij} &\sim \mathcal{U}\left(-(1-\kappa)  K^{-1/2}, (1-\kappa)K^{-1/2}\right)\\
        \bm{h} &= \bm{0},        
    \end{aligned}
\end{equation}
and the observation model is initialized as 
\begin{equation}
    \begin{aligned}
        B_{ij} &\sim \mathcal{U}\left(-(1-\kappa) L^{-1/2}, (1-\kappa) L^{-1/2}\right) \\
        V_{ij} &\sim \mathcal{U}\left(-(1-\kappa)  M^{-1/2}, (1-\kappa)M^{-1/2}\right)\\
        \bm{b} &= \bm{0}.        
    \end{aligned}
\end{equation}
For all experiments, we use $\phi(\cdot) = \mathrm{ReLU}(\cdot)=\max(0,\cdot)$.

\subsection{Loss and regularization}
\paragraph{Manifold attractor regularization (MAR)} For the nonlinear RNN introduced in Eq.~\eqref{eq:shplrnn_nonlinear}, we can encourage a slow manifold in the last $M_r$ units by adding the term \citep{schmidt_identifying_2021}
\begin{equation}\label{eq:mar_loss_latent}
    \mathcal{L}_{\mathrm{MAR}}(\bm{\theta}_{\mathrm{RNN}}) = \frac{\lambda_{\mathrm{MAR}}}{M_r}  \sum_{i=M-M_r+1}^M \left[ \lvert 1-A_{ii}\rvert^p + \frac{1}{L}\sum_{j=1}^L\left( \lvert W_{ij}\rvert ^p + \lvert V_{ji}\rvert^p\right) + \lvert h_i\rvert^p\right]
\end{equation}
to the loss function, where $\lambda_\mathrm{MAR}$ is a regularization parameter and $p \in \{1, 2\}$ determines the type of penalty.
In practice, we found it beneficial to scale regularization terms of connectivity matrices $\bm{W}$ and $\bm{V}$ based on their variance at initialization, i.e.
\begin{equation}\label{eq:mar_loss_latent_initscale}
    \mathcal{L}_{\mathrm{MAR}}(\bm{\theta}_{\mathrm{RNN}}) = \frac{\lambda_{\mathrm{MAR}}}{M_r}  \sum_{i=M-M_r+1}^M \left[ \lvert 1-A_{ii}\rvert^p + \frac{1}{L}\sum_{j=1}^L\left( \gamma_{\bm{W}}\lvert W_{ij}\rvert ^p + \gamma_{\bm{V}}\lvert V_{ji}\rvert^p\right) + \lvert h_i\rvert^p\right],
\end{equation}
where $\gamma_{\bm{W}} = \frac{1}{3L}$,  $\gamma_{\bm{V}} = \frac{1}{3M}$.
For $\lambda_{\mathrm{MAR}} \to \infty$, this regularization leads to stable error propagation by imposing a block-diagonal identity structure in the model Jacobian $\bm{J}_F$, as we demonstrate below.

Recall the nonlinear RNN from Eq.~\eqref{eq:shplrnn_nonlinear} with $\bm{A} = \mathrm{diag}(a_1, \ldots, a_M)$:
\begin{equation}
    \bm{z}_t = \bm{A}\bm{z}_{t-1} + \bm{W}\phi(\bm{V}\bm{z}_{t-1} + \bm{b}) + \bm{C}\bm{s}_t + \bm{h}.
\end{equation}
Its Jacobian is
\begin{equation}
    \bm{J}_F(\bm{z}_{t-1}) = \bm{A} + \bm{W}\,\mathrm{diag}\!\left(\phi'(\bm{V}\bm{z}_{t-1} + \bm{b})\right)\bm{V}.
\end{equation}

Let $M_s := M - M_r$ and partition $\bm{z} = (\bm{z}^{(s)}, \bm{z}^{(r)})^\top$ into
``fast'' (unregularized) and ``slow'' (regularized) units. The diagonal $\bm{A}$
and the non-linear coupling term decompose as
\begin{equation}
    \bm{A} = \begin{pmatrix} \bm{A}_s & \bm{0} \\ \bm{0} & \bm{A}_r \end{pmatrix},
    \qquad
    \bm{W} = \begin{pmatrix} \bm{W}_s \\ \bm{W}_r \end{pmatrix},
    \qquad
    \bm{V} = \begin{pmatrix} \bm{V}_s & \bm{V}_r \end{pmatrix}.
\end{equation}

As $\lambda \to \infty$, the MAR penalty, Eq.~\eqref{eq:mar_loss_latent}, drives
\begin{equation}
    \bm{A}_r \to \bm{I}_{M_r}, \qquad \bm{W}_r \to \bm{0}, \qquad \bm{V}_r \to \bm{0}.
\end{equation}
Writing $\bm{D}_s(\bm{z}_{t-1}) := \mathrm{diag}\!\left(\phi'(\bm{V}_s \bm{z}^{(s)}_{t-1} + \bm{b})\right)$, the limiting Jacobian is block-diagonal:
\begin{equation}
    \bm{J}_F(\bm{z}_{t-1}) \;\xrightarrow{\lambda \to \infty}\;
    \begin{pmatrix}
        \bm{A}_s + \bm{W}_s \bm{D}_s(\bm{z}_{t-1}) \bm{V}_s & \bm{0} \\[4pt]
        \bm{0} & \bm{I}_{M_r}
    \end{pmatrix}.
\end{equation}

The latent space decouples into a nonlinear subsystem acting on $\bm{z}^{(s)}$
and a manifold-attractor subsystem on $\bm{z}^{(r)}$ with identity dynamics.
Consequently, the Jacobian product for a length $t$ sequence and $\lambda \to \infty$ satisfies
\begin{equation}
    \prod_{k=0}^{t-1} \bm{J}_F(\bm{z}_{t-1-k}) = \begin{pmatrix}
        \prod_{k=0}^{t-1}\left[\bm{A}_s + \bm{W}_s \bm{D}_s(\bm{z}_{t-1-k}) \bm{V}_s\right] & \bm{0} \\[4pt]
        \bm{0} & \bm{I}_{M_r}
    \end{pmatrix},
\end{equation}
where error signals propagate without decay or amplification along the $M_r$
directions, enabling stable capture of long-term dependencies.

\paragraph{Observation model regularization} In addition to MAR, we reduce the direct contribution of the $M_r$ MAR units to the read-out and hence discourage them from being forced through Eq.~\eqref{eq:error_forcing} by
\begin{equation}\label{eq:mar_loss_readout}
    \mathcal{L}_{1}(\bm{B}) = \frac{\lambda_1}{N\cdot M_r}  \sum_{i=1}^N\sum_{j=M-M_r}^M \lvert B_{ij}\rvert ^p,
\end{equation}
where $\bm{B}$ is the readout matrix (cf.\ Eq.~\eqref{eq:shplrnn_nonlinear}). We also follow \citep{hess_generalized_2023} and regularize $\bm{B}$ to stay well-conditioned by pulling its singular values towards $1$:
\begin{equation}
    \mathcal{L}_{2}(\bm{B}) = \frac{\lambda_2}{r}\sum_{i=1}^r\left(\sigma_i(\bm{B})-1\right)^2,
\end{equation}
where $r = \operatorname{rank}(\bm{B})$ and $\sigma_i(\bm{B})$ denotes the $i$-th singular value of $\bm{B}$.

\paragraph{Total loss}
The overall loss function used for training of the nonlinear RNN defined in Eq.~\eqref{eq:shplrnn_nonlinear} is given by 
\begin{equation}\label{eq:rnn_loss}
\begin{aligned}
    \mathcal{L}(\bm{x}_{T_w+1:T}, \hat{\bm{x}}_{T_w+1:T};\ \bm{\theta}_{\mathrm{RNN}}, \bm{B}) =& \ \mathcal{L}_{\mathrm{MSE}}(\bm{x}_{T_w+1:T}, \hat{\bm{x}}_{T_w+1:T}) \\ &+ \mathcal{L}_{\mathrm{MAR}}(\bm{\theta}_{\mathrm{RNN}})+ \mathcal{L}_1( \bm{B)} + \mathcal{L}_2(\bm{B}),
\end{aligned}
\end{equation}
with
\begin{equation}\label{eq:mse_loss_warmup}
    \mathcal{L}_{\mathrm{MSE}}(\bm{x}_{T_w+1:T}, \hat{\bm{x}}_{T_w+1:T}) = \frac{1}{N(T-T_w)} \sum_{t=T_w+1}^T \norm{\bm{x}_t - \hat{\bm{x}}_t}_2^2.
\end{equation}

\paragraph{MAR for the LSSM}
Application of MAR to LSSMs (Eq.~\eqref{eq:shplrnn_ssm}) is straightforward with
\begin{equation}\label{eq:lssm_mar_loss}
    \mathcal{L}_{\mathrm{MAR}}(\bm{\theta}_{\mathrm{LSSM}}) = \frac{\lambda_{\mathrm{MAR}}}{M_r}  \sum_{i=M-M_r+1}^M \left[ \lvert 1-A_{ii}\rvert^p + \frac{1}{N}\sum_{j=1}^N\lvert U_{ij}\rvert^p + \lvert h_i\rvert^p\right].
\end{equation}
\subsection{GTF forcing strength}\label{appx:alpha_choice}
There are several methods for choosing $\alpha$ in practice \citep{hess_generalized_2023}. One option is to use \textit{adaptive GTF (aGTF)}, a curriculum learning strategy described in \cite{hess_generalized_2023}, which starts at high forcing subsequently annealed using an exponential moving average. The target $\alpha$ of this annealing schedule is the minimum $\alpha$ required to mitigate exploding gradients and is estimated on the fly from proxies of model Jacobian products.
Alternatively, the simpler option is to perform a grid or line search over $\alpha$ values, where usually $\alpha \in (0, 0.5)$ \citep{emonds2026data, hess_generalized_2023, volkmann2024scalable}. This search may be initialized using (an estimate of) the maximum Lyapunov exponent $\lambda_{\max}$ of the data \citep{mikhaeil_difficulty_2022}, where the minimum forcing strength is given by $\alpha = 1-e^{-\lambda_{\max}}$ \citep{eisen2026disruption}.
 
However, \citep{hess_generalized_2023} provides no guidelines on how to use aGTF in the partially forced setting in this work. Hence, we performed grid search to choose $\alpha$ and report chosen $\alpha$ values in the respective experimental details section (Appx. \ref{appx:experimental_details}). Nonetheless, we also tested aGTF on the forced Lorenz-96 dataset, where most of the dynamics ($N=6$) are observed, and only the sinusoidal forcing is hidden. Using the \texttt{product\_upper\_bound} strategy from the codebase of \citep{hess_generalized_2023}, we found that reconstruction performance was similar (or even slightly improved) to the one obtained by determining $\alpha$ through grid search, with $D_{\mathrm{stsp}}^{\mathrm{DE}} = 3.7$ compared to $3.8\pm0.2$ in Tab.~\ref{tab:model_comparisons}, where aGTF converged to a slightly lower $\alpha$ value of $\alpha \approx 0.055$ compared to $\alpha=0.08$ as found by grid search.

\section{Datasets}\label{appx:datasets}

\paragraph{Lorenz-63}
Our first benchmark system is the three-dimensional Lorenz-63 system~\citep{lorenz_deterministic_1963}. Its dynamics are described by the differential equations
\begin{align}
    \frac{dx}{dt} &= \sigma(y-x),\notag\\
    \frac{dy}{dt} &= x(\rho-z)-y,\\
    \frac{dz}{dt} &= xy-\beta z.\notag
\end{align}
We chose the classical parameters $\sigma=10$, $\rho=28$, $\beta = 8/3$, which put the system into a chaotic regime. We used a Runge-Kutta 4/5 scheme with $\Delta t = 0.01$ and integrated a trajectory of length $100{,}000$ time steps from an initial point $\bm{x}_0$. The trajectory was standardized per dynamical variable after generation.

\paragraph{Forced Lorenz-96} As a benchmark for long time scales, we equipped the vanilla Lorenz-96 equation \citep{lorenz_predictability_1996} with a sinusoidal forcing term:
\begin{equation}\label{eq:forced_lorenz96}
    \frac{dx_i}{dt} = (x_{i+1} - x_{i-2})\,x_{i-1} - x_i + F_0 + A\sin(\omega t),
    \qquad i = 1, \ldots, N,
\end{equation}
where $F_0$ is a constant forcing offset, $A$ is the amplitude and $\omega$ the frequency of the sinusoidal forcing. We set $F_0 = 14$, $A=12$, $\omega = 2\pi \ / \ 75$ and $N=6$. The autonomous Lorenz-96 system (i.e. $A=0$) is chaotic for $F\geq 8$, such that the sinusoidal forcing leads to a periodic switching between cyclic and chaotic dynamics (see Fig. \ref{fig:lor96_example}). For integration we used a Runge-Kutta 4/5 scheme with $\Delta t = 5 \cdot 10^{-3}$ and integrated two trajectories from different initial $t$ and $\bm{x}_0$ of length $500$ internal time units, leading to two trajectories of length $100{,}000$ time steps. One trajectory is contaminated with $5\%$ Gaussian observation noise and used for training, while the other trajectory is kept clean and used for testing. Both trajectories are standardized per dynamical variable after generation.

\paragraph{Bursting neuron model}
We used a three-dimensional simplified Hodgkin-Huxley-type neuron model; its dynamical variables are the membrane potential $V$ as well as two gating variables $n$ and $h$ which control the opening of fast and slow potassium channels, respectively \cite{durstewitz_implications_2009, schmidt_identifying_2021}:
\begin{equation}\label{eq:spiking_diff_eqs}
\begin{aligned}
\dot V &= \frac{1}{C}\Bigl[
I
- g_L (V-E_L)
- g_{\mathrm{Na}}\, m_\infty(V)\,(V-E_{\mathrm{Na}})
- g_K\, n\,(V-E_K) \\
&\hphantom{=\frac{1}{C}\Bigl[}
- g_M\, h\,(V-E_K)
- g_{\mathrm{NMDA}}\, s_\infty(V)\, (V-E_{NMDA})
\Bigr],\\[4pt]
\dot n &= \frac{n_\infty(V)-n}{\tau_n},\\
\dot h &= \frac{h_\infty(V)-h}{\tau_h},
\end{aligned}
\end{equation}
with
\begin{equation}
\begin{aligned}
m_\infty(V) &= \frac{1}{1+\exp\bigl((V_{h,\mathrm{Na}}-V)/k_{\mathrm{Na}}\bigr)}, & h_\infty(V) &= \frac{1}{1+\exp\bigl((V_{h,M}-V)/k_M\bigr)}, \\
n_\infty(V) &= \frac{1}{1+\exp\bigl((V_{h,K}-V)/k_K\bigr)}, &
s_\infty(V) &=\frac{1}{1+0.33\,\exp\bigl(-0.0625\, V\bigr)}.
\end{aligned}
\end{equation}
The model parameters we used for our experiments are reported in Tab.~\ref{tab:parameters_neuron_model}. Similar to the Lorenz-96 system, we used a Runge-Kutta 4/5 scheme with $\Delta t = 2.5 \cdot 10^{-2}$ and integrated two trajectories from different initial conditions of length $5000$ internal time units from which the first $1000$ were removed as transients, leading to two trajectories of length $160{,}000$ discrete time steps. One trajectory is contaminated with $5\%$ Gaussian observation noise and used for training, while the other trajectory is kept clean and used for testing, and both trajectories are standardized per dynamical variable after generation. Moreover, to make reconstruction more challenging, we throw away the slow $h$ variable after generation, such that the system is partially observed with $N=2$. For an example reconstruction, see Fig.~\ref{fig:fig3}.
\begin{table}[h]
\centering
\caption{Neuron model parameter settings}
\label{tab:parameters_neuron_model}
\resizebox{\linewidth}{!}{%
\begin{tabular}{cccccccccccccccccccc}
\toprule
$I$ &
$C$ &
$g_L$ &
$E_L$ &
$g_{\mathrm{Na}}$ &
$E_{\mathrm{Na}}$ &
$V_{h,\mathrm{Na}}$ &
$k_{\mathrm{Na}}$ &
$g_K$ &
$E_K$ &
$V_{h,K}$ &
$k_K$ &
$\tau_n$ &
$g_M$ &
$V_{h,M}$ &
$k_M$ &
$\tau_h$ &
$g_{\mathrm{NMDA}}$ &
$E_{\mathrm{NMDA}}$  \\
\midrule
$0$ &
$6$ &
$8$ &
$-80$ &
$20$ &
$60$ &
$-20$ &
$15$ &
$10$ &
$-90$ &
$-25$ &
$7$ &
$1$ &
$25.2$ &
$-18$ &
$5$ &
$1000$ &
$10.2$ &
$0$ \\
\bottomrule
\end{tabular}}
\end{table}

\paragraph{Deng--Hines predator--prey system.} 
The predator--prey system models a tritrophic Rosenzweig--MacArthur-type food chain with Holling type II responses \cite{hastings1991chaos} and a fast prey variable, following \cite{deng2001food, deng2002food}. It is described by equations 
\begin{equation}
\begin{aligned} 
   \zeta\,\dot x &= x\Bigl(1-x-\frac{y}{b_1+x}\Bigr),\\
\dot y &= y\Bigl(\frac{x}{b_1+x}-d_1-\frac{z}{b_2+y}\Bigr),\\
\dot z &= \varepsilon\,z\Bigl(\frac{y}{b_2+y}-d_2\Bigr),
   \end{aligned}
\end{equation}
with the chaotic parameter set $\zeta=0.05$, $\varepsilon=1$, $b_1=0.25$, $b_2=0.1$, $d_1=0.2$, $d_2=0.39$ and initial condition $(x,y,z)=(0.8,0.3,9)$. The prey $x$ is fast ($\zeta\ll 1$); $y$ and $z$ evolve on a common $\mathcal{O}(1)$ timescale. The stiff system was integrated with the implicit Radau method ($\mathrm{rtol}=10^{-8}$, $\mathrm{atol}=10^{-10}$, maximum step $0.05$) and sampled at $\Delta t=0.01$. We generated a trajectory of length $540{,}000$ time steps, of which we discarded $40{,}000$ time steps of transient dynamics. We then retained $250{,}000$ steps for training and the following $250{,}000$ for testing. Both time series were standardized with the training mean and standard deviation, and contaminated with i.i.d.\ Gaussian observation noise of standard deviation $0.05$. Only the prey coordinate ($x$) is observed.

\paragraph{Neurophysiological recordings.} To test the algorithm on long real-world time series, whole-cell patch clamp recordings of the membrane potential of single rodent entorhinal cortex neurons in-vivo were used \citep{hahn2012spontaneous}. As described in the main text, these involve slow and irregular state transitions between hyperpolarized ``DOWN'' and depolarized ``UP'' states carrying fast spikes, thus multiple timescales from milliseconds to many seconds, as well as considerable noise levels (see Fig.~\ref{fig:real_world_data_rollout}). To deal with the noise and stiff spikes in the ``UP'' states, we applied a Gaussian smoothing filter with standard deviation of $\sigma=30$ with a kernel length of $241$ time points. Additionally, we center and scale the trace by its median and median absolute deviation and then apply an inverse hyperbolic sine,
\begin{equation}
    x^{(\mathrm{trf})}_{1:T_{tot}} = \sinh^{-1}\left( \frac{x_{1:T_{tot}} - \operatorname{median}(x_{1:T_{tot}})}{\operatorname{MAD}(x_{1:T_\mathrm{tot}})} \right)
\end{equation}
which leaves the slow up–down fluctuations nearly unchanged while compressing large spikes. The total length of the dataset is $20 \mathrm{s}$, which corresponds to $100{,}000$ samples at a sampling rate of $5 \mathrm{kHz}$.

\section{Evaluation measures}\label{appx:measures}
\paragraph{State space divergence $D_{\textrm{stsp}}$}\label{appx:dstsp}
The state space divergence $D_{\textrm{stsp}}$ measures the geometrical disagreement between state distributions of the data $p(\bm{x})$ and that of model generated trajectories $q(\bm{x})$ by evaluating the Kullback-Leibler (KL) divergence:
\begin{align}\label{eq:D_stsp}
    D_{\textrm{stsp}} := D_{\textrm{KL}}\left(p(\bm{x}) \ || \ q(\bm{x})\right) = \int_{\bm{x} \in \mathbb{R}^N} p(\bm{x}) \log\frac{p(\bm{x})}{q(\bm{x})} d\bm{x}.
\end{align} 
In practice, we estimate $p(\bm{x})$ and $q(\bm{x})$ by placing Gaussian Mixture Models (GMM) along orbits \citep{brenner_tractable_2022, koppe_identifying_2019}), i.e.
\begin{equation}
    \begin{aligned}
        \hat{p}(\bm{x}) &= \frac{1}{T_1}\sum_{t=1}^T \mathcal{N}(\bm{x};\ \bm{x}_t, \bm{\Sigma}) \\
        \hat{q}(\bm{x}) &= \frac{1}{T_2}\sum_{t=1}^{T_2} \mathcal{N}(\bm{x}; \ \hat{\bm{x}}_t, \bm{\Sigma}),
    \end{aligned}
\end{equation}
where $\bm{x}_{1:T_1}$ is the ground-truth data of length $T_1$ and $\hat{\bm{x}}_{1:T_2}$ is a model-generated trajectory of length $T_2$, $\mathcal{N}(\bm{x}; \ \bm{x}_t, \bm{\Sigma})$ is a multivariate Gaussian with mean vector $\bm{x}_t$ and covariance matrix $\bm{\Sigma} = \mathrm{diag}([\sigma^2_1, \dots, \sigma^2_N])$. The KL-divergence between two GMMs can be approximated using a Monte-Carlo approach \citep{Hershey2007ApproximatingTK}
\begin{align}
D_{\textrm{stsp}} = D_{\textrm{KL}}\left(\hat{p}(\bm{x}) \ || \ \hat{q}(\bm{x})\right) \approx \frac{1}{n} \sum_{i=1}^n \log \frac{\hat{p}(\bm{x}^{(i)})}{\hat{q}(\bm{x}^{(i)})},
\end{align}
with $n$ Monte Carlo samples $\bm{x}^{(i)}$ randomly drawn from the GMM based on observed orbits, $\hat{p}(\bm{x}^{(i)})$. When the observed system is only partially observed and the true dynamics unfold in $d > N$ dimensions, the attractor first has to be unfolded to apply $D_{\mathrm{stsp}}$ correctly \citep{botvinick2025invariant}. To this end we first apply a delay embedding (DE) \citep{sauer_embedology_1991, takens_detecting_1981} of the observed and model generated orbits
\begin{equation}
    \bm{\xi}_t = \left[ \bm{x}_t, \bm{x}_{t-\tau}, \bm{x}_{t-2\tau}, \dots, \bm{x}_{t-(m-1)\tau} \right]^\top  \in \mathbb{R}^{mN}
\end{equation}
with delay $\tau$ and embedding dimension $m$, and the same construction performed for $\hat{\bm{\xi}}_t$. We then fit GMMs and calculate $D_{\mathrm{stsp}}$ in embedding space, i.e.
\begin{equation}\label{eq:delay_dstsp}
    D^{\mathrm{DE}}_{\mathrm{stsp}} = D_{\textrm{KL}}\left(\hat{p}(\bm{\xi}) \ || \ \hat{q}(\bm{\xi})\right) \approx \frac{1}{n} \sum_{i=1}^n \log \frac{\hat{p}(\bm{\xi}^{(i)})}{\hat{q}(\bm{\xi}^{(i)})}.
\end{equation}

For all experiments, we use $n=10^6$ and set the bandwidth $\bm{\Sigma} = f_{bw} \cdot \mathrm{diag}([\hat{\sigma}_1^2, \dots , \hat{\sigma}_d^2])$, where $f_{bw} = \left(\frac{T_1(d+2)}{4}\right)^{-1/(d+4)}$ is a bandwith factor chosen according to Silverman's rule-of-thumb \citep{silverman2018density} and $\hat{\sigma}_i$ is the empirical standard deviation of $x_{i, 1:T_1}$. Furthermore, to probe long-term consistency beyond the available data, we choose to generate model orbits for $T_2 = 3 \cdot T_1$. This avoids low $D_{\mathrm{stsp}}$ when the model is only \textit{transiently} generating the correct long-term behavior. 

Since most systems are only partially observed, we use $D_{\mathrm{stsp}}^{\mathrm{DE}}$ (\eqref{eq:delay_dstsp}) with parameters $\tau$ and $m$ given per dataset in Tab.~\ref{tab:dstsp_de_settings}. These were chosen to specifically resolve long time scales in the data, except for the partially observed Lorenz-63, which does not exhibit particularly long time scales.
\begin{table}[h!]
  \centering
  \caption{Delay embedding parameters used for each dataset.}
  \label{tab:embedding_params}
  \begin{tabular}{lcc}
    \toprule
    \textbf{Dataset} & $\bm{\tau}$ & $\bm{m}$ \\
    \midrule
    Lorenz-63 (1D) & $10$ & $3$ \\
    Forced Lorenz-96 (6D)  & $4{,}096$ & $3$ \\
    Bursting Neuron (2D)   & $1{,}024$ & $7$ \\
    Deng--Hines (1D)   & $350$ & $5$ \\
    Neural Recordings (1D) & $1{,}634$ & $5$ \\
    \bottomrule
  \end{tabular}
  \label{tab:dstsp_de_settings}
\end{table}

\paragraph{Short-term RMSE}
The $n$-step-ahead RMSE assesses short- to medium-length forecasts and is defined as
\begin{equation}
    \text{RMSE}(n) = \sqrt{\frac{1}{n \cdot N} \sum_{k=1}^{n} \left\lVert\bm{x}_k - G_{\bm{\psi}}(F_{\bm{\theta}}^{\circ k}(\bm{z}_0))\right\rVert^2}
\end{equation}
where $\bm{x}_{1:n} := \bm{x}_{t_i:t_i+n-1}$ is a window starting at index $t_i$ of the total ground-truth data $\bm{x}_{1:T_{\mathrm{tot}}}$. The initial condition $\bm{z}_0$ is obtained by performing a warm-up using a history of length $T_w$, $\bm{x}_{t_i-T_w:t_i-1}$ (see Sect.~\ref{subsec:longterm_dependencies} for details on warm-up). In practice we averaged the RMSE over $100$ windows from different initial conditions.

\section{Additional experimental details}\label{appx:experimental_details}
This appendix section lists specific experimental details of experiments performed in the main paper. All experiments were performed on a compute server with 2x96-Core AMD EPYC 9655 CPUs, 768 GB RAM, and 6x NVIDIA RTX Pro 6000 Blackwell (96 GB) GPUs; however, individual trainings were always performed on a single GPU at a time. 

\subsection{Figure \ref{fig:1}}\label{appx:details_fig1}
\paragraph{Fig. 1\textbf{A}} We trained shPLRNNs (Eq.~\eqref{eq:shplrnn_nonlinear}, $L=50$) of varying latent dimensions $M\in \{4, 16, 64, 128\}$ on the Lorenz-63 system using GTF-DEER. The batch size was fixed at $B=1$, while the sequence length was varied in powers of 2, i.e. $T \in \{128, 256, \, \dots, \ 32{,}768\}$ and we used no warm-up for this experiment ($T_w=0$). We measured the time of combined forward+backward passes throughout training. Training was performed for $10,000$ updates (samples). We chose to explicitly train a model to account for the fact that GTF-DEER's convergence depends on the dynamics of the model, giving us an estimate of the spread of runtimes throughout training. Indeed, for the chosen $\alpha = 0.15$, GTF-DEER's convergence stayed consistent throughout training (see also Fig. \ref{fig:qdeer_ablation}\textbf{B}). The sequential baseline runtime is independent of explicit dynamics, and hence we only measured the time of trained models using $N_s=500$ samples. In code, the only difference between the sequential and parallel implementation is the call to the forward pass solver, which is given by a straight-forward \texttt{jax.lax.scan} for the sequential case, and the \texttt{seq1d} (DEER solver) function for the parallel one, respectively. Both methods parallelize over the batch size using \texttt{jax.vmap}. All data were gathered on an NVIDIA RTX PRO 6000 Blackwell (96GB) GPU.

\paragraph{Fig. 1\textbf{B}} For this figure we used a shPLRNN (varying $M$, fixed $L=50$) that was trained on the full Lorenz-63 system ($N=3$). We performed forward passes for different values of forcing strength $\alpha \in [0, 1]$. The maximum number of Newton iterations was capped at $n_{\textrm{iter}}=500$: In practice, GTF-DEER loses its competitive performance over sequential evaluation even for moderate Newton iterations such that $500$ is already a generous upper bound of what can be considered a useful regime. The different curves correspond to different initial guess strategies, where `Zero' corresponds to initializing $\bm{z}^{(0)}_{1:T} = \bm{0}_{T \times M}$. `$\mathcal{N}(0,1)$' draws entries from a standard Normal distribution, $z_{i, t}^{(0)}\sim \mathcal{N}(0,1)$ and `P-inv' uses $\bm{z}^{(0)}_{1:T} = \bm{B}^+\bm{x}_{1:T}$ to initialize the Newton iterations.

\subsection{Figure \ref{fig:qdeer_ablation}}
For this experiment we trained shPLRNNs ($M=5$, $L=50$) on the Lorenz-63 under different conditions. We trained on the full Lorenz-63 system ($N=3$) as well as on the partially observed one ($N=1$, $x$-component). Furthermore, we switched between GTF-DEER with full Jacobians (Eq.~\eqref{eq:GTF_DEER_update}) and diagonalized ones (quasi-DEER, $\bm{J}_{F} \rightarrow \operatorname{diag}(\bm{J}_{F})$). We trained $20$ models per setting for $20{,}000$ parameter updates. The forcing strength was $\alpha=0.15$.

\subsection{Section \ref{subsec:recurrence_comparison}}
\subsubsection{Figure $\ref{fig:fig3}\textbf{A}$}\label{appx:fig_3a_settings}
We trained shPLRNNs with $M=10$ and $L=128$ on the forced Lorenz-96 system and $M=6$ for the bursting neuron model (see Appx. \ref{appx:datasets}). To test the effect of increasing sequence length on this problem, we trained the models on sequence lengths $T\in\{256, 512, 1{,}024, 2{,}048, 4{,}096, 8{,}192, 16{,}384, 32{,}768\}$, where for each setting we used $T_w = T/2$. To keep data parsed by the model per parameter update the same, we adjusted batch size accordingly, such that the product $B\cdot T=2^{15} =32{,}768$ staid fixed. This resulted in settings $(B, T) = (256, 128)$ and $(B, T) = (1, 32{,}768)$ for the lowest and highest sequence length used, respectively. For each data point in the figure, we trained $10$ independent models. For the Lorenz-96 we used a forcing strength of $\alpha=0.08$, for the bursting neuron model $\alpha=0.4$.

\subsubsection{Mamba-2}\label{appx:mamba2}
As a state-of-the-art SSM baseline, we compare the LSSM
(Eq.~\eqref{eq:shplrnn_ssm}) and the GTF-DEER-trained
shPLRNN (Eq.~\eqref{eq:shplrnn_nonlinear}) to a model in
which the linear recurrence of the LSSM is replaced by a single Mamba-2 block
\citep{dao2024transformers}, while keeping the rest of the architecture, in particular the one-hidden-layer MLP read-out from Eq.~\eqref{eq:shplrnn_ssm}, identical. We refer to this configuration as Mamba-2-$n_p$, where $n_p$ is the number of parameters. For the
non-recurrence machinery of the Mamba-2 block (input projections, depth-wise causal convolution, gating branch, multi-head setup, output projection) we refer the reader to the original work by \citep{dao2024transformers}. In the following, we clarify the--in our minds important--distinction to the LSSM introduced in Eq.~\eqref{eq:shplrnn_ssm}.

\paragraph{Selective SSM recurrence.} The main \emph{functional} difference of the Mamba recurrence is the introduction of a \emph{selection mechanism} \citep{dao2024transformers, gu2024mamba}, which renders the model non-autonomous even in the absence of external inputs $\bm{s}_t$. Dropping $\bm{s}_t$ for brevity and using the LSSM (Eq.~\eqref{eq:shplrnn_ssm}) as a template, the core recurrence inside the Mamba-2 block can be written as
\begin{equation}
    \bm{z}_t = \bar{\bm{A}}_t(\bm{x}_{t-1}) \, \bm{z}_{t-1} + \bar{\bm{U}}_t(\bm{x}_{t-1}) \, \bm{x}_{t-1}
    \label{eq:mamba_recurrence}
\end{equation}
i.e., \emph{structurally} identical to the LSSM but with $\bar{\bm{A}}_t$ and $\bar{\bm{U}}_t$ now depending on the input $\bm{x}_{t-1}$. Specifically, parameterizations of $\bar{\bm{A}}_t$ and $\bar{\bm{U}}_t$ are given by
\begin{equation}
    \begin{aligned}\label{eq:mamba_discretization}
    \Delta_t &= \mathrm{softplus}(\bm{W}_\Delta \bm{x}_{t-1} + \bm{b}_\Delta) \in \mathbb{R} \\
    \bm{U}_t &= \bm{W}_U \bm{x}_{t-1} \\
    \bar{\bm{A}}_t &= \exp(\Delta_t \, \bm{A}) \\
    \bar{\bm{U}}_t &= \Delta_t \, \bm{U}_t,
    \end{aligned}
\end{equation}
where $\bm{A}$ is a \emph{fixed}, learnable parameter which for Mamba-2 is restricted to $\bm{A} = a\,\bm{I}$ for a single learned scalar $a < 0$ per head. The recurrence parameters thus become functions of the input through the data-dependent scalar $\Delta_t$ and the data-dependent input vector $\bm{U}_t$, while $\bm{A}$ itself
remains time-invariant. Setting $\Delta_t = \Delta$ and
$\bm{U}_t = \bm{U}$ recovers a fixed-parameter linear recurrence equivalent to the LSSM (Eq.~\eqref{eq:shplrnn_ssm}).\footnote{
For clarity, Eqs.~\eqref{eq:mamba_recurrence} \& \eqref{eq:mamba_discretization}
describe the Mamba-2 recurrence at the level of abstraction relevant for
comparison with the LSSM. In practice, a Mamba-2 block does not operate
directly on $\bm{x}_{t-1}$: the input is first lifted to an expanded
representation of dimension $D = E \cdot N$ via a linear projection, where $E$ is the expansion factor, then split into $H$ heads, and each head runs an independent SSM with its own data-dependent parameters $(\Delta_t, \bm{U}_t)$ shared across the channels within that head, while $\bm{A}$ is a learned scalar per head.
Channel mixing across the observation dimensions therefore occurs in the
`peripheral' input/output projections rather than inside the recurrence
itself. We refer the reader to \citep{dao2024transformers} for the full
block specification.}

\paragraph{Relation to the LSSM} Crucially, the recurrence in
Eq.~\eqref{eq:mamba_recurrence} is still \emph{linear in the latent state} $\bm{z}_{t-1}$ for any fixed value of $\bm{x}_{t-1}$, which is what permits parallelization via parallel scan during training \citep{dao2024transformers}.
At the same time, the selection mechanism circumvents the
$\mathrm{rank}\le\min(N, M, L)$ bottleneck on the effective test-time recurrence (Eq.~\eqref{eq:shplrnn_ssm_test_time}) discussed in Sect.~\ref{subsec:autoregressive_dsr}: because $\bm{U}_t$ varies with $t$, the model is no longer constrained to a single fixed product $\overline{\bm{W}} = \bm{U}\bm{B}$ of fixed rank, but can in principle
realize a different effective input pathway at every step.

\subsubsection{Fig. 3\textbf{B} hyperparameters}
\paragraph{Optimization settings} For experiments in Sect.~\ref{subsec:recurrence_comparison} we used the following shared hyperparameters: For the forced Lorenz-96 experiments in \ref{fig:fig3}\textbf{B}, we used batch size $B=1$, sequence length $T=81{,}920$ with warm-up $T_w=16{,}384$ and a learning rate decay from $\eta_s=5\cdot10^{-5}$ to $10^{-6}$ using a cosine decay schedule~\citep{loshchilov2017sgdr}. Training was performed for $150{,}000$ parameter updates.

\paragraph{Mamba-2} For the Mamba-2 models, we evaluated two parameter budgets, denoted Mamba-2-7k with $7{,}038$ and Mamba-2-17k with $16{,}830$ parameters, respectively. The former is matched as closely as possible to the shPLRNN/LSSM budget ($\approx 2{,}200$--$2{,}900$ parameters) 
for direct comparison at fixed capacity without overly limiting Mamba-2's expressiveness; the latter probes whether additional
capacity closes the gap. Both use a single Mamba-2 block with hidden dimension
$D = 16$ (\texttt{d\_model}) for Mamba-2-7k and $D=32$ for Mamba-2-17k. All other parameters stay the same: State size $N = 16$ (\texttt{d\_state}), $\texttt{d\_conv}=4$ and $\texttt{expand}=2$. The Mamba-2 block is followed by the same MLP read-out as used in the LSSM (Eq.~\eqref{eq:shplrnn_ssm}) with $L=128$. We used the Mamba-2 implementation from the original code repository \url{https://github.com/state-spaces/mamba}, licensed under the Apache License, Version 2.0.
\paragraph{RNN}
The RNN used $M=10$ and $L=128$ and $M_r=4$. We used $\lambda_{\mathrm{MAR}}=1$, $\lambda_1 = 1$ and $\lambda_2=10^{-4}$, however, we found the regularization of the singular values of $\bm{B}$ did not have a huge impact on training convergence and reconstruction performance. The forcing was $\alpha=0.08$.

\paragraph{LSSM}
For the LSSM we kept core settings equal to the ones used in the RNN, i.e. $M=10$, $L=128$. Additionally, we apply MAR to the LSSM (see Appx. \ref{appx:regularizations}) where we scanned $\lambda_{\mathrm{MAR}} \in \{10^{-4}, 10^{-3}, 10^{-2}, 10^{-1}, 1, 10\}$ and found that all settings produced similar results, except $\lambda_{\mathrm{MAR}} = 10$ which degraded performance. Hence we reported $\lambda_{\mathrm{MAR}}=1$ to make clear that MAR does not improve reconstruction for the LSSM. We further used $M_r=4$ in this case, mirroring the setting of the RNN. 

\subsubsection{Predator--prey and neurophysiological data}

\paragraph{Predator-Prey.} For the 1D predator--prey system, we trained shPLRNNs ($M=10, L=128)$ with GTF-DEER using $\alpha=0.15$. For regularization losses, we used $M_r=9$ with $\lambda_{\mathrm{MAR}} = 0.1$, $\lambda_1=20$ and $\lambda_2 = 10^{-2}$. We used a learning rate of $\eta = 10^{-4}$, which decayed using a cosine schedule to reach $10^{-5}$ at the end of training. For the LSSM, we used $M=20$ and $L=128$, matching the parameter count of the shPLRNNs. We also started at a slightly higher learning rate of $\eta = 5 \cdot 10^{-4}$.
The Mamba-2 model was the same as the small model used in experiments on the forced Lorenz-96, i.e. a single Mamba-2 block with $D = 16$ (\texttt{d\_model}), state size $N = 16$ (\texttt{d\_state}), $\texttt{d\_conv}=4$ and $\texttt{expand}=2$, followed by an MLP with hidden size $L=128$.
For all models, parameter updates were calculated based on single sequences ($B=1$) with common length $T=49{,}152$, of which $T_W = 16{,}384$ are used for warm-up. For the shPLRNNs specifically, we also trained models with $T = 1{,}536$ and $B=32$ (i.e. the same amount of total data per update), showing that the training benefits massively from sequences which span a long contiguous segment of the observed data.

\paragraph{Neurophysiological data.} For the physiological voltage recordings from \citep{hahn2012spontaneous}, we trained shPLRNNs ($M=15, L=128)$ with GTF-DEER using $\alpha=0.09$. For regularization losses, we used $M_r=14$ with $\lambda_{\mathrm{MAR}} = 0.01$, $\lambda_1=0.1$ and $\lambda_2 = 10^{-4}$. We also applied randomly drawn isotropic Gaussian noise to teacher forcing signals after sampling the training sequence with standard deviation $\sigma=0.025$, which we found to stabilize training for this data. We used a learning rate of $\eta = 5\cdot10^{-5}$, which decayed using a cosine schedule to reach $5\cdot10^{-6}$ at the end of training. Additionally, we used a mean absolute value (MAE) loss, which we found to provide reconstructions of higher fidelity.

For the LSSM, we used $M=29, L=128$ to match the parameter count of the shPLRNN. We also applied MAR to all states and L2 regularization to the observation model with strength $\lambda=0.01$, which we found through line search in $\lambda \in [10^{-5}, 10^{-4}, \dots,  10^{-1}]$. We used the same learning rate schedule as for the shPLRNN trained through GTF-DEER.

Since the Mamba-2 model profited from more parameters in Tab.~\ref{tab:model_comparisons}, this higher parameter count was also used in the present experiment. The model is again comprised of a single Mamba-2 block with hidden dimension $D = 32$ (\texttt{d\_model}), state size $N = 16$ (\texttt{d\_state}), $\texttt{d\_conv}=4$ and $\texttt{expand}=2$. The Mamba-2 block is followed by an output MLP with hidden size $L=128$.

Finally, all models used a common sequence length $T=50{,}000$, of which $T_w = 10{,}000$ are used for warm-up, and batch size $B=2$.

\section{Additional figures and tables}
\begin{figure}
    \centering
    \includegraphics[width=0.8\linewidth]{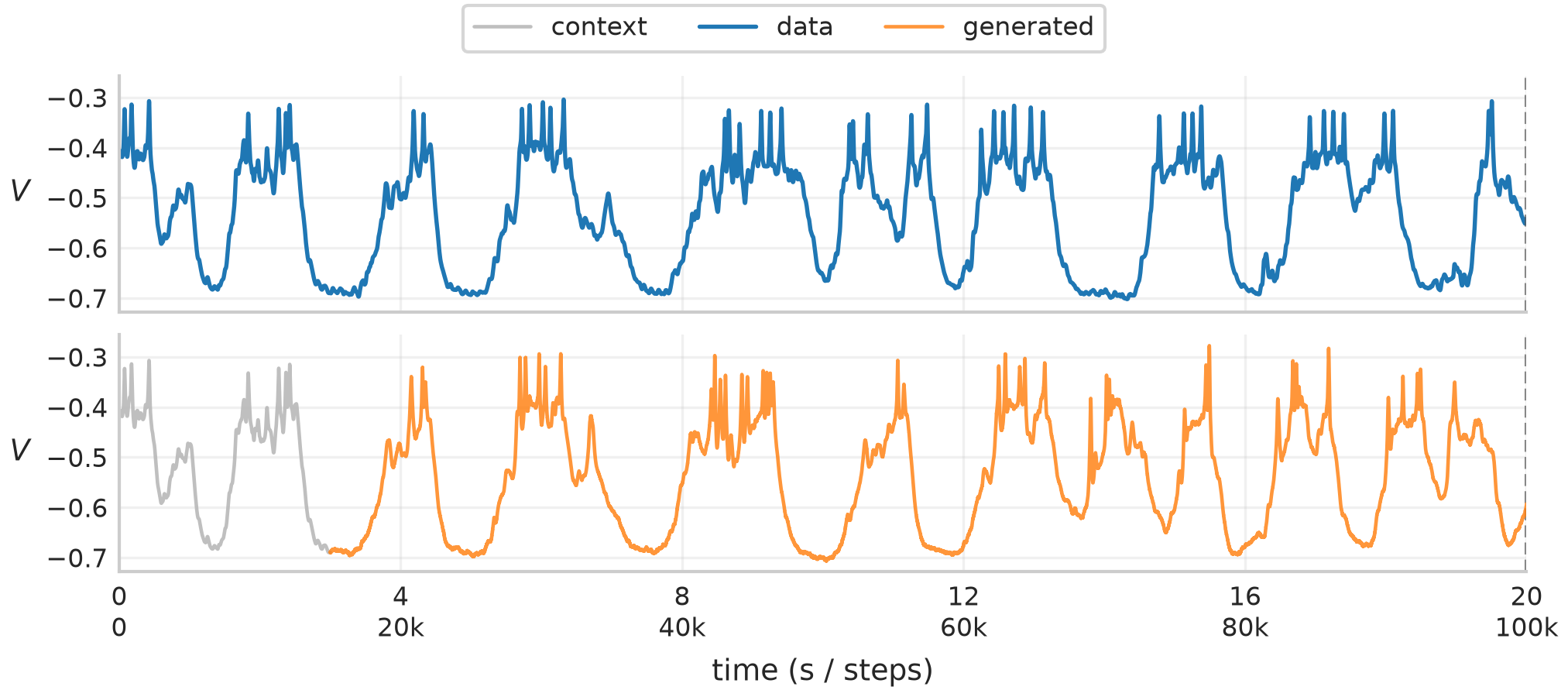}

    \caption{Example rollout of a shPLRNN trained through GTF-DEER (bottom) on physiological neuron recordings (top) \citep{hahn2012spontaneous}. To mimic process noise present in the data, we applied isotropic Gaussian process noise during the roll-out with strength $\sigma=10^{-3}$, i.e. $\bm{z}_t=F_{\bm{\theta}}(\bm{z}_{t-1}) + \sigma \bm{\epsilon}_t$, where $\bm{\epsilon}_t \overset{\mathrm{iid}}{\sim} \mathcal{N}(\bm{0}, \bm{I}_{M \times M}).$}
    \label{fig:real_world_data_rollout}
\end{figure}
\newcommand{\hl}{\cellcolor{gray!18}}
\begin{table}[h!]
  \centering
  \caption{Model comparisons on two simulated and one real-world benchmark. Columns are $T$: sequence length, $B$: batch size,
            long-term measure $D^\mathrm{DE}_\mathrm{stsp}$, short-term measure $\mathrm{RMSE}(128)$, and number of parameters $\lvert \bm{\theta}\rvert$. All values are median $\pm$ MAD over 10 independent runs. While (d-)LSSMs and Mamba-2 models are trained using conventional TF, shPLRNNs are trained using GTF-DEER (highlighted in gray). Bold values denote the best result per dataset and metric within overlapping $\pm$ MAD intervals.}
  \label{tab:model_comparisons}
  \begin{tabular}{p{2.3cm}llrrccc}
    \toprule
    \textbf{Dataset} & \textbf{Model} & \textbf{$T$} & \textbf{$B$} & \textbf{$D^\mathrm{DE}_\mathrm{stsp}$} & RMSE(128) & $\lvert \bm{\theta}\rvert$\\
    \midrule
    \multirow{9}{=}{Forced Lorenz-96 (6D)}
      & LSSM & $81{,}920$ & $1$ & $8.8 \pm 1.8 $ & $0.26 \pm 0.01$ & $2{,}256$ \\
      & LSSM + MAR & $81{,}920$ & $1$ & $9.1 \pm 2.2 $ & $0.41 \pm 0.02$ & $2{,}256$  \\
      & d-LSSM & $81{,}920$ & $1$ & $7.8 \pm 0.9 $ & $0.27 \pm 0.01$ & $2{,}346$ \\
      & d-LSSM + MAR & $81{,}920$ & $1$ & $9.9 \pm 3.8 $ & $0.32 \pm 0.01$ & $2{,}346$ \\
      & Mamba-2-7k & $81{,}920$ & $1$ & $10.5 \pm 2.6 $ & $0.21 \pm 0.02$ & ${7,038}$ \\
      & Mamba-2-17k & $81{,}920$ & $1$ & $8.7 \pm 0.8 $ & $\mathbf{0.14} \pm \mathbf{0.02}$ & $16{,}830$ \\
      & \hl shPLRNN ($r=6$) &\hl  $81{,}920$ & \hl $1$ & \hl $7.7 \pm 0.3 $ &\hl  $0.19 \pm 0.01$ & \hl $2{,}316$  \\
      & \hl  shPLRNN ($r=7$) & \hl $81{,}920$ & \hl $1$ & \hl $\mathbf{3.8} \pm \mathbf{0.2}$ &\hl  $\mathbf{0.16} \pm \mathbf{0.01}$ & \hl $2{,}454$  \\
      & \hl shPLRNN & \hl $81{,}920$ & \hl $1$ & \hl $4.3 \pm 0.2$ & \hl $0.17 \pm 0.01$ & \hl $2{,}868$  \\
    \midrule
    \multirow{4}{=}{Deng--Hines (1D)}
      & LSSM & $49{,}152$ & $1$ & $15 \pm 12$ & $0.22 \pm 0.13$ & $2{,}876$\\
      & Mamba-2 & $49{,}152$ & $1$ & $34.5 \pm 0.5$ & $0.093 \pm 0.001$ & 4{,}761 \\
      & \hl shPLRNN & \hl $1{,}536$ & \hl $32$ & \hl $36 \pm 3$ & \hl $11 \pm 10$ & $2{,}718$\\
      & \hl shPLRNN & \hl $49{,}152$ & \hl $1$ & \hl $\mathbf{0.6} \pm \mathbf{0.4}$ & \hl $\mathbf{0.076} \pm \mathbf{0.012}$ & $2{,}718$\\
    \midrule
    \multirow{3}{=}{Neural recordings (1D)}
      & LSSM & $50{,}000$ & $2$  & $6.2 \pm 2.5$ & $\mathbf{0.068} \pm \mathbf{0.002}$ & $4{,}055$ \\
      & Mamba-2 & $50{,}000$  & $2$ & $50.3 \pm 0.6$ & $\mathbf{0.068} \pm \mathbf{0.001}$ & $16{,}505$ \\
      & \hl shPLRNN & \hl $50{,}000$  & \hl $2$  & \hl $\mathbf{0.68} \pm \mathbf{0.02}$ & \hl $\mathbf{0.071} \pm \mathbf{0.008}$ & \hl $4{,}013$\\
    \bottomrule
  \end{tabular}
\end{table}

\begin{figure}[htb!]
    \centering
    \includegraphics[width=1.0\linewidth]{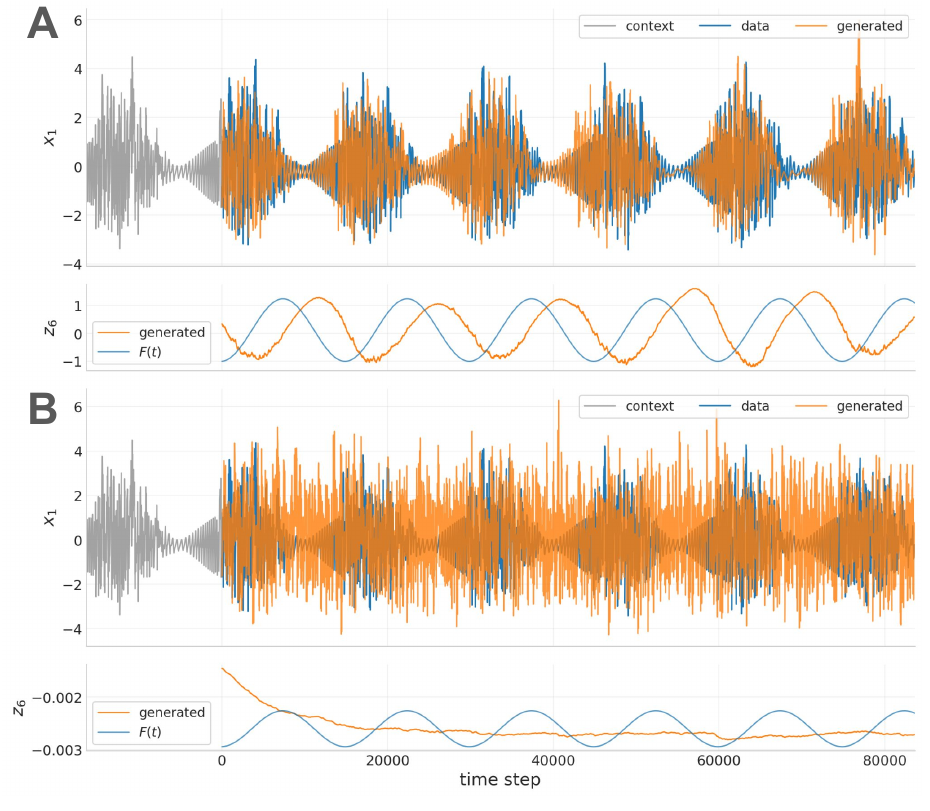}
    \caption{Generated dynamics of a shPLRNN ($M=10, L=128$) trained on the forced Lorenz-96 system. \textbf{A}: Access to long sequences during training enables the model to capture the latent long-period sinusoidal forcing $F(t)$ in the MAR regularized units, and hence to produce accurate long-term rollouts. \textbf{B}: A model trained on short sequence lengths w.r.t. the intrinsic time scales of the data fails to capture and reproduce the latent forcing during auto-regressive roll-outs.}
    \label{fig:lor96_example}
\end{figure}

\begin{figure}[htb!]
    \centering
    \includegraphics[width=0.9\linewidth]{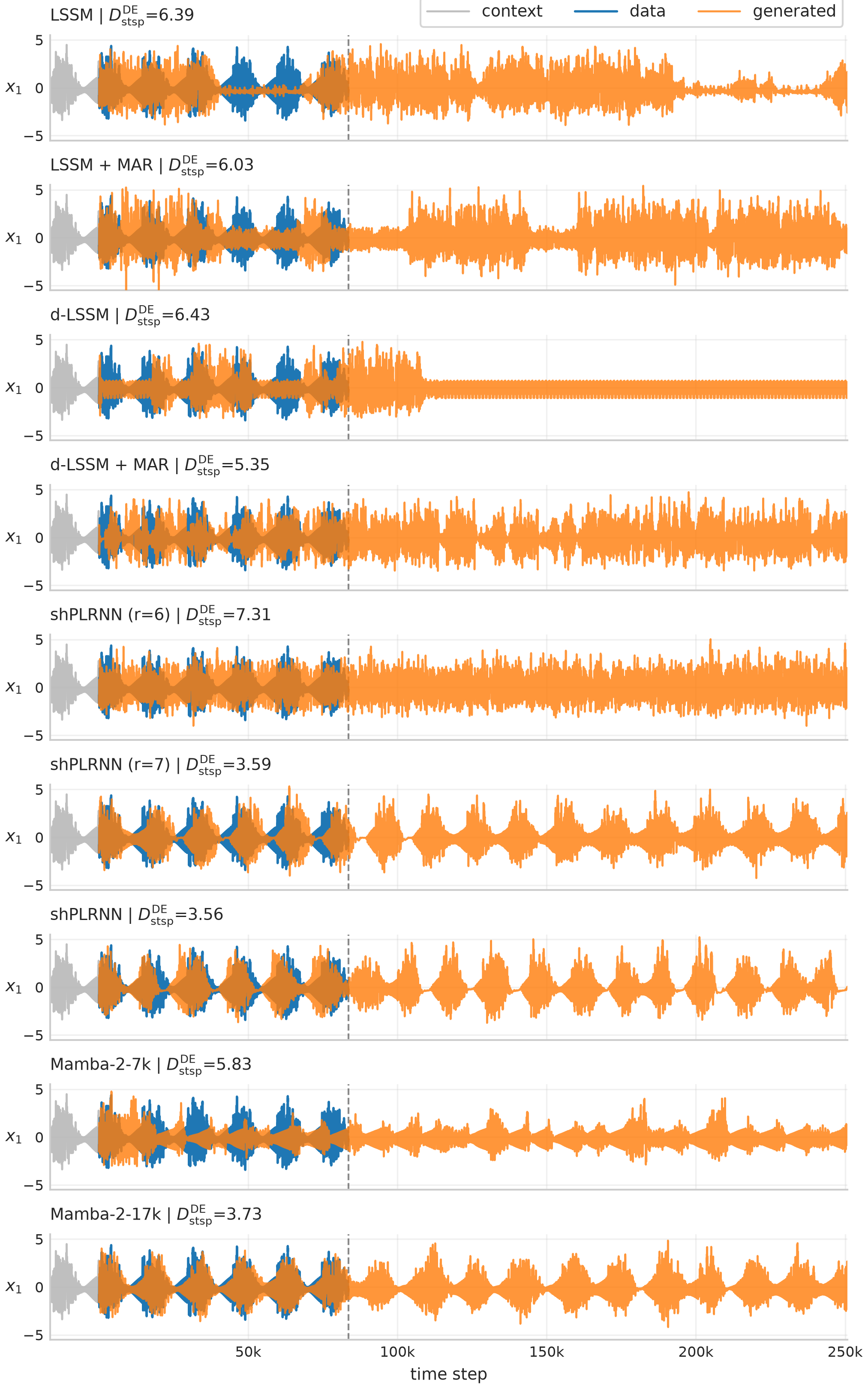}
    \caption{Generated dynamics across $3\times$ the length of available ground truth data from models evaluated in Sect. \ref{subsec:recurrence_comparison}. shPLRNNs with their rank limited to the number of observed variables $r=N=6$ and all LSSM variants fail to learn the latent sinusoidal forcing. Increasing the rank to $7$ in the shPLRNN enables the model to learn the underlying forcing and to produce convincing limiting behavior; however, the model still struggles to infer the correct phase from the context. The Mamba-2 based model benefits from more parameters and can achieve good reconstructions (Mamba-2-17k), but lacks the robustness of the shPLRNN trained through GTF-DEER (see Fig. \ref{fig:fig3}\textbf{B} and Tab.~\ref{tab:model_comparisons}). Each generated trace was produced using the best model out of $10$ per configuration.}
    \label{fig:model_comparison_traces}
\end{figure}
\begin{figure}[htb!]
    \centering
    \includegraphics[width=1.0\linewidth]{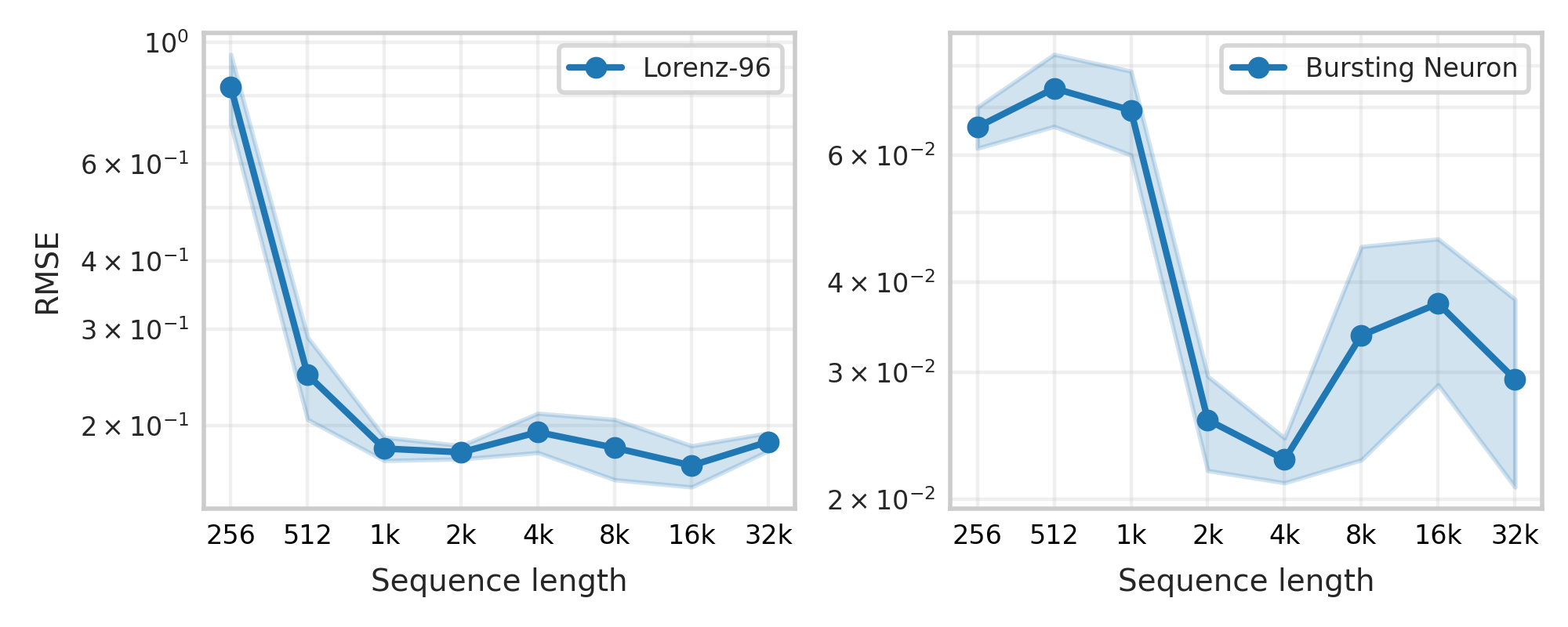}
    \caption{$\textrm{RMSE}(128)$ as a function of sequence length for shPLRNNs trained on the forced Lorenz-96 and bursting neuron systems, complementary to the $D_\mathrm{stsp}$ plot in Fig.~\ref{fig:fig3}.  Curves show Median $\pm$ MAD, $10$ runs per setting. The $\textrm{RMSE}(128)$ already saturates for shorter sequence lengths, as expected by design of this short-term measure (mean $\pm$ SEM over all sequence lengths and runs: $0.27 \pm 0.02$ for the Lorenz-96 and $0.047 \pm 0.003$ for the bursting neuron system, respectively.)}
    \label{fig:rmse_fig3}
\end{figure}

\FloatBarrier
\section{Training algorithms}
\begin{algorithm}[H]
\caption{Linear SSM training (using parallel associative scan)}
\label{alg:LSSM}
\begin{algorithmic}[1]
\linespread{1.4}\selectfont
\Require Data $\bm{X} \in \mathbb{R}^{T_\mathrm{obs}\times N}$, optional external inputs $\bm{S} \in \mathbb{R}^{T_\mathrm{obs}\times K}$, linear recurrence $F_{\bm{\theta}}(\bm{z}_{t-1}, \bm{x}_{t-1}, \bm{s}_t)$, decoder $G_{\bm{\psi}}(\bm{z}_t)$, batch size $B$, sequence length $T$, warm-up length $T_w$
\Ensure Trained parameters $\bm{\theta}, \bm{\psi}$
\Repeat

    \State Sample sequences $\bm{x}_{0:T} \in \mathbb{R}^{B\times (T+1)\times N}$ from $\bm{X}$, $\bm{s}_{1:T} \in \mathbb{R}^{B\times T\times K}$ from $\bm{S}$ \LComment{Implicit vectorized batch processing below, e.g.\ via \texttt{jax.vmap}}
    \State Initialize $\bm{z}_0 = \bm{0}$
    \State $\bm{z}_{1:T} \gets \mathrm{associative\_scan}(F_{\bm{\theta}}, \bm{z}_0, \bm{x}_{0:T-1}, \bm{s}_{1:T})$ \Comment{Solve in parallel on GPU}
    \State $\hat{\bm{x}}_{1:T} = G_{\bm{\psi}}(\bm{z}_{1:T})$
    \State $\mathcal{L} \gets \mathrm{MSE}(\bm{x}_{T_w+1:T}, \hat{\bm{x}}_{T_w+1:T})$ \Comment{Exclude warm-up steps from loss}
    \State $(\bm{g}_{\bm{\theta}}, \bm{g}_{\bm{\psi}}) \gets \mathrm{grad}(\mathcal{L}; \ \bm{\theta}, \bm{\psi}$)
    \State $(\bm{\theta}, \bm{\psi}) \gets \mathrm{optimize(\bm{\theta}, \bm{\psi}, \bm{g}_{\bm{\theta}}, \bm{g}_{\bm{\psi}}})$
\Until{convergence}
\end{algorithmic}
\end{algorithm}

\begin{algorithm}[H]
\caption{Initial value / trajectory matching training with generalized teacher forcing}
\label{alg:ivp}
\begin{algorithmic}[1]
\linespread{1.4}\selectfont
\Require Data $\bm{X} \in \mathbb{R}^{T_\mathrm{obs}\times N}$, optional external inputs $\bm{S} \in \mathbb{R}^{T_\mathrm{obs}\times K}$, recurrent model $F_{\bm{\theta}}$, observation matrix $\bm{B} \in \mathbb{R}^{N \times M}$, forcing strength $\alpha \in [0, 1]$, batch size $B$, sequence length $T$
\Ensure Trained parameters $\bm{\theta}, \bm{B}$
\Repeat
    \State Compute $\bm{B}^+ \gets \mathrm{pinv}(\bm{B})$, \quad $\bm{P} \gets \bm{I}_M - \alpha \bm{B}^+ \bm{B}$

    \State Sample sequences $\bm{x}_{0:T} \in \mathbb{R}^{B\times (T+1)\times N}$ from $\bm{X}$, $\bm{s}_{1:T} \in \mathbb{R}^{B\times T\times K}$ from $\bm{S}$ \LComment{Implicit vectorized batch processing below, e.g.\ via \texttt{jax.vmap}}
    \State Compute teacher signals $\bar{\bm{z}}_{0:T} = \bm{B}^+(x_{0:T})$
    \State Initialize $z_0 \gets \bar{z}_0$
    \For{$t = 1, \dots, T$}
        \State $\tilde{\bm{z}}_{t-1} \gets \bm{P}\, \bm{z}_{t-1} + \alpha\, \bar{\bm{z}}_{t-1}$ \Comment{Build weighted forced state}
        \State $\bm{z}_t \gets F_{\bm{\theta}}(\tilde{\bm{z}}_{t-1}, \bm{s}_{t})$
        \State $\hat{\bm{x}}_t \gets \bm{B}\bm{z}_t$
    \EndFor
    \State $\mathcal{L} \gets \mathrm{MSE}(\bm{x}_{1:T}, \hat{\bm{x}}_{1:T})$
    \State $(\bm{g}_{\bm{\theta}}, \bm{g}_{\bm{B}}) \gets \mathrm{grad}(\mathcal{L}; \ \bm{\theta}, \bm{B})$ \Comment{Backpropagation through time}
    \State $(\bm{\theta}, \bm{B}) \gets \mathrm{optimize}(\bm{\theta}, \bm{B}, \bm{g}_{\bm{\theta}}, \bm{g}_{\bm{B}})$
\Until{convergence}
\end{algorithmic}
\end{algorithm}

\begin{algorithm}[H]
\caption{GTF-DEER: Generalized teacher forcing with DEER parallel solver}
\label{alg:gtf-deer}
\begin{algorithmic}[1]
\linespread{1.4}\selectfont
\Require Data $\bm{X} \in \mathbb{R}^{T_{\mathrm{obs}}\times N}$, optional external inputs $\bm{S} \in \mathbb{R}^{T_\mathrm{obs}\times K}$, recurrent model $F_{\bm{\theta}}$, observation matrix $\bm{B} \in \mathbb{R}^{N \times M}$, forcing strength $\alpha \in [0, 1]$, batch size $B$, sequence length $T$, tolerance $\varepsilon$
\Ensure Trained parameters $\bm{\theta}$, $\bm{B}$
\Repeat
    \State Compute $\bm{B}^+ \gets \mathrm{pinv}(\bm{B})$, \quad $\bm{P}_\alpha \gets \bm{I}_M - \alpha \bm{B}^+ \bm{B}$, \quad $\bm{P}_1 \gets \bm{I}_M - \bm{B}^+ \bm{B}$
    \State Sample sequences $\bm{x}_{0:T} \in \mathbb{R}^{B\times (T+1)\times N}$ from $\bm{X}$, $\bm{s}_{1:T} \in \mathbb{R}^{B\times T\times K}$ from $\bm{S}$ \LComment{Implicit vectorized batch processing below, e.g.\ via \texttt{jax.vmap}}
    \State Compute forcing targets $\bar{\bm{z}}_t \gets \bm{B}^+ \bm{x}_t$ for $t = 0, \ldots, T$
    \State Initialize $\bm{z}_0 \gets \bar{\bm{z}}_0$ \Comment{Below, fix $\bm{z}^{(k)}_0 = \bm{z}_0 \;\forall k$}
    \State Initial guess $\bm{z}^{(0)}_{1:T} \gets \bar{\bm{z}}_{1:T}$ \Comment{Warm start from data}
    \For{$k = 0, 1, \ldots$ \textbf{until} $\|\Delta\bm{z}_{1:T}^{(k+1)}\|_\infty < \varepsilon$} \Comment{DEER iteration}
        \State $\tilde{\bm{z}}_t^{(k)} \gets \bm{P}_1\bm{z}_{t}^{(k)}+\bar{\bm{z}}_t$ for $t = 1, \ldots, T_w$ \Comment{Warm-up}
        \State $\tilde{\bm{z}}_t^{(k)} \gets \bm{P}_\alpha \bm{z}^{(k)}_{t} + \alpha\bar{\bm{z}}_t$ for $t = T_w + 1, \ldots, T$ \Comment{Forced inputs}
        \State $\bm{J}_t^{(k)} \gets \frac{\partial F_\theta}{\partial \bm{z}}\big|_{\tilde{\bm{z}}_t^{(k)}} \bm{P}_1$ for $t = 1, \ldots, T_w$ \Comment{Warm-up Jacobians}
        \State $\bm{J}_t^{(k)} \gets \frac{\partial F_\theta}{\partial \bm{z}}\big|_{\tilde{\bm{z}}_t^{(k)}} \bm{P}_\alpha$ for $t = T_w + 1, \ldots, T$ \Comment{Forced Jacobians}
        \State $\bm{r}_t^{(k)} \gets \bm{z}^{(k)}_t - F_{\bm{\theta}}(\tilde{\bm{z}}_{t-1}^{(k)})$ for $t = 1, \ldots, T$ \Comment{Residuals}
        \State Solve $\Delta\bm{z}_t^{(k+1)} = \bm{J}_{t-1}^{(k)} \Delta\bm{z}_{t-1}^{(k+1)} - \bm{r}_t^{(k)}$ for $t = 1, \ldots, T$ via associative scan
        \State $\bm{z}^{(k+1)}_{1:T} \gets \bm{z}^{(k)}_{1:T} + \Delta\bm{z}^{(k+1)}_{1:T}$
    \EndFor
    \State $\hat{\bm{x}}_{T_w+1:T} \gets \bm{B} \bm{z}_{T_w+1:T}^{(\mathrm{final})}$
    \State $\mathcal{L} \gets \mathrm{MSE}(\bm{x}_{T_w+1:T}, \hat{\bm{x}}_{T_w+1:T})$ \Comment{Loss excluding warm-up states}
    \State $(\bm{g}_{\bm{\theta}}, \bm{g}_{\bm{B}}) \gets \mathrm{grad}(\mathcal{L}; \ \bm{\theta}, \bm{B})$ \Comment{Single associative scan via IFT}
    \State $(\bm{\theta}, \bm{B}) \gets \mathrm{optimize}(\bm{\theta}, \bm{B}, \bm{g}_{\bm{\theta}}, \bm{g}_{\bm{B}})$
\Until{convergence}
\end{algorithmic}
\end{algorithm}

\clearpage
\section*{NeurIPS Paper Checklist}

\begin{enumerate}

\item {\bf Claims}
    \item[] Question: Do the main claims made in the abstract and introduction accurately reflect the paper's contributions and scope?
    \item[] Answer: \answerYes{} 
    \item[] Justification: The abstract clearly states the challenges that motivated our contributions, the main idea we introduced to address these challenges, and our conclusion about the benefits of our method drawn from our experiments. All claims have structured counterparts in the main text, which faithfully backs them up with evidence.

\item {\bf Limitations}
    \item[] Question: Does the paper discuss the limitations of the work performed by the authors?
    \item[] Answer: \answerYes{} 
    \item[] Justification: Limitations, in particular regarding scaling with latent dimension and scope of our proposition, are stated in the main text and are again summarized in a dedicated paragraph in the Discussion section.

\item {\bf Theory assumptions and proofs}
    \item[] Question: For each theoretical result, does the paper provide the full set of assumptions and a complete (and correct) proof?
    \item[] Answer: \answerYes{} 
    \item[] Justification: The relevant definitions and assumptions for our proposition are clearly stated in the main text, and a proof is provided in the appendix.

    \item {\bf Experimental result reproducibility}
    \item[] Question: Does the paper fully disclose all the information needed to reproduce the main experimental results of the paper to the extent that it affects the main claims and/or conclusions of the paper (regardless of whether the code and data are provided or not)?
    \item[] Answer: \answerYes{} 
    \item[] Justification: The model architectures, algorithms, hyperparameters, regularizations, datasets, and evaluation measures needed for reproducing our results are all clearly and comprehensively described in the main text and appendices.

\item {\bf Open access to data and code}
    \item[] Question: Does the paper provide open access to the data and code, with sufficient instructions to faithfully reproduce the main experimental results, as described in supplemental material?
    \item[] Answer: \answerYes{}{} 
    \item[] Justification: For the review process, the anonymized paper version includes a link to an anonymous repository with our code. We will make our code publicly available upon publication.

\item {\bf Experimental setting/details}
    \item[] Question: Does the paper specify all the training and test details (e.g., data splits, hyperparameters, how they were chosen, type of optimizer) necessary to understand the results?
    \item[] Answer: \answerYes{} 
    \item[] Justification: Hyperparameter settings for model training, as well as training and evaluation methods, are comprehensively documented in the main text and appendices.

\item {\bf Experiment statistical significance}
    \item[] Question: Does the paper report error bars suitably and correctly defined or other appropriate information about the statistical significance of the experiments?
    \item[] Answer: \answerYes{} 
    \item[] Justification: Where applicable, figures include error bars, and the text states confidence intervals. Together with the detailed description of our evaluation measures in the appendix, these elements allow for the critical confirmability of our judgment on the statistical significance of our results.

\item {\bf Experiments compute resources}
    \item[] Question: For each experiment, does the paper provide sufficient information on the computer resources (type of compute workers, memory, time of execution) needed to reproduce the experiments?
    \item[] Answer: \answerYes{} 
    \item[] Justification: Details on our compute resources are stated in the appendix on additional experimental details.
    
\item {\bf Code of ethics}
    \item[] Question: Does the research conducted in the paper conform, in every respect, with the NeurIPS Code of Ethics \url{https://neurips.cc/public/EthicsGuidelines}?
    \item[] Answer: \answerYes{} 
    \item[] Justification: Our research conforms, in every respect, with the NeurIPS Code of Ethics.

\item {\bf Broader impacts}
    \item[] Question: Does the paper discuss both potential positive societal impacts and negative societal impacts of the work performed?
    \item[] Answer: \answerNA{} 
    \item[] Justification: Our methods for dynamical systems reconstruction will be valuable to the scientific community working in this field, but we do not expect any broader societal impact.
    
\item {\bf Safeguards}
    \item[] Question: Does the paper describe safeguards that have been put in place for responsible release of data or models that have a high risk for misuse (e.g., pre-trained language models, image generators, or scraped datasets)?
    \item[] Answer: \answerNA{} 
    \item[] Justification: All of our datasets were generated using mathematical equations, so there is no potential for misuse.

\item {\bf Licenses for existing assets}
    \item[] Question: Are the creators or original owners of assets (e.g., code, data, models), used in the paper, properly credited and are the license and terms of use explicitly mentioned and properly respected?
    \item[] Answer: \answerYes{} 
    \item[] Justification: We cite the papers that describe the methods and code upon which our work is based, as well as the papers associated with the code used for comparisons. Where applicable, the appendix on additional experimental details provides URLs and information on licenses under which models have been made available.

\item {\bf New assets}
    \item[] Question: Are new assets introduced in the paper well documented and is the documentation provided alongside the assets?
    \item[] Answer: \answerYes{} 
    \item[] Justification: The code provided alongside the paper is documented.

\item {\bf Crowdsourcing and research with human subjects}
    \item[] Question: For crowdsourcing experiments and research with human subjects, does the paper include the full text of instructions given to participants and screenshots, if applicable, as well as details about compensation (if any)? 
    \item[] Answer: \answerNA{} 
    \item[] Justification: No crowdsourcing experiments or research with human subjects were conducted.

\item {\bf Institutional review board (IRB) approvals or equivalent for research with human subjects}
    \item[] Question: Does the paper describe potential risks incurred by study participants, whether such risks were disclosed to the subjects, and whether Institutional Review Board (IRB) approvals (or an equivalent approval/review based on the requirements of your country or institution) were obtained?
    \item[] Answer: \answerNA{} 
    \item[] Justification: Our paper involves no crowdsourcing or research with human subjects.

\item {\bf Declaration of LLM usage}
    \item[] Question: Does the paper describe the usage of LLMs if it is an important, original, or non-standard component of the core methods in this research? Note that if the LLM is used only for writing, editing, or formatting purposes and does \emph{not} impact the core methodology, scientific rigor, or originality of the research, declaration is not required.
    \item[] Answer: \answerNA{} 
    \item[] Justification: Our core method development in this research does not involve LLMs.

\end{enumerate}

\end{document}